%% file: main.tex
\documentclass[final]{cvpr}

\usepackage[utf8]{inputenc} % allow utf-8 input
\usepackage[T1]{fontenc}    % use 8-bit T1 fonts
\usepackage{amsfonts}       % blackboard math symbols
\usepackage[toc,page]{appendix}

\usepackage{float}
\usepackage{enumitem}
\setlist[enumerate]{itemsep=0mm}

\input{math_commands.tex}

\usepackage{url}
\usepackage{caption}
\usepackage{array}
\usepackage{makecell}
\usepackage{bbm}
\usepackage[pdftex]{graphicx}
\usepackage{booktabs}
\usepackage{tabularx}
\usepackage{amsmath}
\usepackage{amsthm}
\usepackage{amssymb}
\usepackage{mathtools}      % provides amsmath + goodies like mathclap
\usepackage{microtype}      % microtypography
\usepackage{nicefrac}       % compact symbols for 1/2, etc.
\usepackage{subcaption} 
\usepackage{url}  

\usepackage{xspace}
\usepackage{amssymb}

\usepackage{multirow}
\usepackage{enumitem}

\usepackage{soul}

\usepackage{amsthm}
\usepackage{bbm}
\usepackage{xcolor}
 
\usepackage{wrapfig}

\usepackage{thmtools,thm-restate}

\newcommand{\revise}[1]{{\color{blue}{#1}}}

\newtheorem{Thm}{Theorem}
\newtheorem{Lm}{Lemma}
\newtheorem{Df}{Definition}

\makeatletter
\newcommand\footnoteref[1]{\protected@xdef\@thefnmark{\ref{#1}}\@footnotemark}
\makeatother

\usepackage{array}
\newcommand{\PreserveBackslash}[1]{\let\temp=\\#1\let\\=\temp}
\newcolumntype{C}[1]{>{\PreserveBackslash\centering}p{#1}}
\newcolumntype{R}[1]{>{\PreserveBackslash\raggedleft}p{#1}}
\newcolumntype{L}[1]{>{\PreserveBackslash\raggedright}p{#1}}

\usepackage{times}
\usepackage{epsfig}
\usepackage{graphicx}
\usepackage{amsmath}
\usepackage{amssymb}

\usepackage[pagebackref=true,breaklinks=true,colorlinks,bookmarks=false]{hyperref}

\def\cvprPaperID{10164} % *** Enter the CVPR Paper ID here
\def\confYear{CVPR 2021}
\begin{document}

%%%%%%%%% TITLE
\title{Scalability vs. Utility: Do We Have to Sacrifice One for the Other in Data Importance Quantification?}

\author{Ruoxi Jia$^{1}$
% {\tt\small ruoxijia@vt.edu}
% For a paper whose authors are all at the same institution,
% omit the following lines up until the closing ``}''.
% Additional authors and addresses can be added with ``\and'',
% just like the second author.
% To save space, use either the email address or home page, not both
\qquad
Fan Wu$^{2}$\thanks{Equal contribution.}
% {\tt\small fanw6@illinois.edu}
\qquad
Xuehui Sun$^{3}$\footnotemark[1]
% {\tt\small zidaneandmessi@sjtu.edu.cn}
\qquad
Jiacen Xu$^{4}$\footnotemark[1]
% {\tt\small jiacenx@uci.edu}
\qquad
David Dao$^{5}$ \\
% {\tt\small david.dao@inf.ethz.ch}
\qquad
Bhavya Kailkhura$^{6}$
% {\tt\small kailkhura1@llnl.gov}
\qquad
Ce Zhang$^{5}$
% {\tt\small ce.zhang@inf.ethz.ch}
\qquad
Bo Li$^{2}$
% {\tt\small lbo@illinois.edu}
\qquad
Dawn Song$^{7}$
% {\tt\small dawnsong@gmail.com}
\\
$^{1}$Virginia Tech
\qquad
$^{2}$UIUC
\qquad
$^{3}$Shanhai Jiaotong University
\qquad
$^{4}$UC Irvine
\qquad
$^{5}$ETH Zurich
\\
\qquad
$^{6}$Lawrence Livermore National Laboratory
\qquad
$^{7}$UC Berkeley
\\
{\tt\small ruoxijia@vt.edu}
\qquad
{\tt\small \{fanw6,lbo\}@illinois.edu}
\qquad
{\tt\small zidaneandmessi@sjtu.edu.cn}
\qquad
{\tt\small jiacenx@uci.edu}
\\
{\tt\small \{david.dao,ce.zhang\}@inf.ethz.ch}
\qquad
{\tt\small kailkhura1@llnl.gov}
\qquad
{\tt\small dawnsong@gmail.com}
}

\maketitle

\vskip 0.3in

%\vspace{-0.5em}
\begin{abstract}
%\vspace{-0.5em}
Quantifying the importance of each training point to a learning task is a fundamental problem in machine learning and the estimated importance scores have been leveraged to guide a range of data workflows such as data summarization and domain adaption. One simple idea is to use the leave-one-out error of each training point to indicate its importance. Recent work has also proposed to use the Shapley value, as it defines a unique value distribution scheme that satisfies a set of appealing properties. 
% So far it has been applied to a range of data workflows such as data summarization and domain adaption. 
However, calculating Shapley values is often expensive, which limits its applicability in real-world applications at scale. 
Multiple heuristics to improve the scalability of calculating Shapley values have been proposed recently, with the potential risk of compromising their utility in real-world applications. 
% One idea is to replace the original classifier with a surrogate $K$-nearest neighbor ($K$NN) classifier, which has a locality structure that can lead to efficient Shapley value calculation. Despite its scalability, such an approach only returns the Shapley value for a surrogate model different from the original one; another idea is to approximate the aggregation over exponentially many possible data subsets in the Shapley semantics by leave-one-out on a single data subset. In spite of its efficiency, the 
% returned value does not satisfy the properties of Shapley values. 

\textit{How well do existing data quantification methods perform on existing workflows? How do these methods compare with each other, empirically and theoretically? Must we sacrifice scalability for the utility in these workflows when using these methods?} In this paper, we conduct a novel theoretical analysis comparing the utility of different importance quantification methods, and report extensive experimental studies on existing and proposed workflows such as noisy label detection, watermark removal, data summarization, data acquisition, and domain adaptation. We show that Shapley value approximation based on a $K$NN surrogate over pre-trained feature embeddings obtains comparable utility with existing algorithms while achieving significant scalability improvement, often by orders of magnitude. Our theoretical analysis also justifies its advantage over the leave-one-out error.

The code is available at \url{https://github.com/AI-secure/Shapley-Study}.
\end{abstract}

%\vspace{-1em}
\section{Introduction}
%\vspace{-1em}

Understanding the \textit{importance} of a single
training example, relative to other
training examples, to a learning task
is a fundamental problem in
machine learning (ML) which could have profound impact
on a range of applications including interpretability, robustness, data acquisition, data valuation, among others~\cite{jia2019towards,ghorbani2019data,koh2017understanding}. 

In this paper, we are driven by \textit{two} questions
around this fundamental problem. Our contribution
is a novel theoretical analysis and thorough experimental studies towards understanding both questions.

{\bf \ul{Q1: Leave-one-out vs. Shapley?}}
Given a training set $D$, a validation set $D_{val}$
and a learning algorithm $\mathcal{A}$, let the
utility $U_{\mathcal{A}, D_{val}}(D)$ be 
the validation accuracy of the model trained
on $D$ using $\mathcal{A}$, recently 
there have been 
two lines of work in assigning 
relative importance to a data point $z \in D$.

{\bf A. Leave-one-out (LOO)-based Method \& Influence Function.} One natural 
way to assign importance to $z$ is by calculating 
its contribution to the rest of training data:
\begin{small}
\begin{align*}
v_{loo}(z) \propto U_{\mathcal{A}, D_{val}}(D) - U_{\mathcal{A}, D_{val}}(D \backslash \{z\})
\end{align*}
\end{small}
When we need to assign such an importance score to all 
data points in the training set, we need to 
train a large number of models. Thus researchers
have proposed efficient techniques to 
approximate this score, e.g., via influence
function~\cite{koh2017understanding}.

{\bf B. Shapley-based Method.} Another natural way
to assign importance to $z$ is inspired by 
cooperative game theory and to use the Shapley value~\cite{jia2019towards,ghorbani2019data}:
\begin{small}
\begin{align*}
\nu_\text{shap}(z) \propto \frac{1}{N} \sum_{S\subseteq D\setminus\{z\}} \frac{1}{{N-1 \choose |S|}}
\big[U_{\mathcal{A}, D_{val}}(S\cup \{z\})-U_{\mathcal{A}, D_{val}}(S)\big]
\end{align*}
\end{small}
Both approaches have recently 
been explored by researchers and have been 
applied to a range of ML tasks including 
noisy label detection, watermark removal,
data summarization, active data acquisition,
and domain adaptation~\cite{jia2019towards,ghorbani2019data,koh2017understanding}. 

However, one question remains: \textit{What's the relationships and differences,
both theoretically and empirically, between these two lines of approaches?}

{\bf \ul{Q2: Exact Shapley vs. 
Shapley over Surrogates?}}
As we will show in this paper,
Shapley-based methods often outperforms 
leave-one-out-based methods, both
theoretically and empirically.
However, Shapley-based approaches can be expensive
as one needs to train, for general classifiers,
exponentially many models. Thus, many state-of-the-art approaches resort to 
a sampling-based approach~\cite{jia2019towards,ghorbani2019data}
to approximate this score. On the other hand,
a recent work by Jia \textit{et al.}~\cite{jia2019efficient}
has shown that for a certain family of classifiers,
i.e., $K$-Nearest Neighbor ($K$NN),
calculating this score can be done efficiently,
in $\OO(|D| \log |D|)$ time for \textit{all} 
data points in $D$. 

Despite this, there still remains a question: \textit{Can we 
use a $K$-Nearest Neighbor classifier as
a surrogate model to calculate the Shapley
value, and how does it perform on real-world
applications compared with the vanilla exact Shapley
value?}

{\bf \ul{Technical Contributions.}}
In this paper, we take the first step
towards understanding the above questions.
We make contributions on both theoretical
and empirical fronts.

%%\vspace{-0.5em}
\begin{itemize}[leftmargin=*]

%\vspace{-0.5em}
\item We conduct a novel theoretical analysis 
aiming at rigorously analyzing the differences
between the leave-one-out-based and
the Shapley-based methods. Specifically,
we formalize two performance metrics specific to data importance: one focuses on the predictive power of data importance, studying whether it is indicative of a training point's contribution to a random set; the other focuses on the ability of a data to discriminate ``good'' training points 
from ``bad'' ones. We show that for both 
performance metrics, under certain 
technical conditions, 
the Shapley-based method can outperform 
leave-one-out--based approaches. To our best knowledge,
this is the first theoretical analysis
reasoning the relative performances of different
data importance quantification techniques.
%
%\vspace{-0.5em}
\item We conduct a thorough empirical study
on a range of ML tasks, including noisy label detection, watermark removal,
data summarization, active data acquisition,
and domain adaptation on different benchmark datasets. Some have been used
by previous work as a use case of data valuation
methods, and some are proposed by us. On these
tasks, we empirically investigate the relative 
performance between (1) leave-one-out--based methods
and Shapley-based methods, and (2) exact Shapley-based
methods and Shapley over $K$NN Surrogates.
%
%\vspace{-0.5em}
\item Our empirical study
suggests that the Shapley-over-$K$NN-Surrogates method 
performs well and achieves comparable results with, and often 
outperforms, all other methods in quality while
being orders of magnitude faster. This gives us 
the first practical algorithm 
over large-scale datasets that returns 
useful data importance scores for a range of 
important ML tasks.
%
%\vspace{-0.5em}
% \item On some tasks and datasets, we even 
% obtained new state-of-the-art quality
% with the $K$NN-Surrogate
% outperforming other more task-specific 
% approaches.
\end{itemize}

\section{Background: General Frameworks for Data Importance Quantification}
%\vspace{-0.5em}

There are two lines of work in estimating 
the importance of a single training point
for supervised learning~\cite{jia2019towards,ghorbani2019data,koh2017understanding}. In this section,
we describe these methods and the $K$NN-surrogate-based method to set the 
context for our analysis in Section~\ref{sec:theory}-\ref{section:exp}.

We first set up the notations to characterize the main ingredients of a supervised learning problem, including the training and validation data, the learning algorithm, and the performance measure. Let $D=\{z_i\}_{i=1}^N$ be the training set, where $z_i$ is a feature-label pair $(x_i,y_i)$, and $D_\text{val}$ be the validation data. Let $\mathcal{A}$ be the learning algorithm which maps a training dataset to a model. Let $U$ be a performance measure which takes as input training data, any learning algorithm, and validation data and returns a score. We write $U(S,\mathcal{A},D_\text{val})$ to denote the performance score of the model trained on a subset $S$ of training data using the learning algorithm $\mathcal{A}$ when testing on $D_\text{val}$. When the learning algorithm and validation data are self-evident, we will suppress the dependence of $U$ on them and just use $U(S)$ for short. Our goal is to assign a score to each training point $z_i$, denoted by $\nu(z_i, D,\mathcal{A},D_\text{val}, U)$, indicating its \textit{importance} to the supervised learning problem specified by $D,\mathcal{A},D_\text{val}, U$. We will often write it as $\nu(z_i)$ or $\nu(z_i,U)$ to simplify notation.

%\vspace{-0.5em}
\subsection{Leave-One-Out Method}
%\vspace{-0.5em}

One simple way to quantify data importance is to measure one data point's contribution to the rest of the training data:
\begin{align}
    \nu_\text{loo}(z_i) = U(D) - U(D\setminus \{z_i\})
\end{align}
This data importance measure is referred to as the Leave-One-Out (LOO) value. The exact evaluation of the LOO values for $N$ training points requires re-training the model for $N$ times and the associated computational cost is prohibitive for large training datasets and large models. For deep neural networks, Koh \textit{et al.}~\cite{koh2017understanding} proposed to estimate the model performance change due to the removal of each training point via influence functions. However, in order to obtain the influence functions, one will need to evaluate the inverse of the Hessian for the loss function. With $N$ training points and $p$ model parameters, it requires $\OO(N p^2 + p^3)$ operations. Koh \textit{et al.}~\cite{koh2017understanding} approximate the influence function with $\OO(Np)$ complexity, which is still expensive for large networks. 

%In contrast, our method, which will be discussed in Section~\ref{section:alg}, has complexity independent of model size, thus preferable for models like deep neural networks.

%\vspace{-0.5em}
\subsection{Shapley Value-based Method }
%\vspace{-0.5em}

The Shapley value is a classic concept in cooperative game theory to distribute the total gains generated by the coalition of all players. One can think of a supervised learning problem as a cooperative game among training data instances and apply the Shapley value to value the contribution of each training point. 
Given a performance measure $U$, the Shapley value for training data $z_i$ is defined as the average marginal contribution of $z_i$ to all possible subsets of $D$ formed by other training points:
\begin{small}
\begin{align}
\label{eqn:shap_def}
    \nu_\text{shap}(z_i) = \frac{1}{N} \sum_{S\subseteq D\setminus\{z_i\}} \frac{1}{{N-1 \choose |S|}}
\big[U(S\cup \{z_i\})-U(S)\big]
\end{align}
\end{small}
However, calculating the Shapley value can be expensive: evaluating the exact Shapley value involves computing the marginal contribution of each training point to all possible subsets, whose complexity is $\mathcal{O}(2^N)$. Such complexity is clearly impractical for valuating a large number of training points. Even worse, for ML tasks, evaluating the utility function \textit{per se} (e.g., validation accuracy) is computationally expensive as it requires re-training an ML model. 

{\bf MCMC-based Approximation.} Ghorbani \textit{et al.}~\cite{ghorbani2019data} introduced two approaches to approximating the Shapley value based on Monte Carlo approximation. The central idea behind these approaches is to treat the Shapley value of a training point as its expected contribution to a random subset and use sample average to approximate the expectation. By the definition of the Shapley value, the random set has size 0 to $N-1$ with equal probability and is also equally likely to be any subset of a given size (corresponding to the $1/ {{N-1}\choose{|S|}}$ factor). In practice, one can implement an equivalent sampler by drawing a random permutation of the training set. Then, the Shapley value can be estimated by computing the marginal contribution of a point to the points preceding it and averaging the marginal contributions across different permutations. However, these Monte Carlo-based approaches cannot circumvent the need to re-train models and therefore are not viable for large models. In our experiments, we found that the approaches in Ghorbani \textit{et al.}~\cite{ghorbani2019data} can manage data size up to one thousand for simple models such as logistic regression and shallow neural networks, while failing to estimate the Shapley value for larger data sizes and deep nets in a reasonable amount of time. We evaluate runtime in more details in Section~\ref{section:exp}.

{\bf $K$NN Surrogate-based Approach.}
In one recent paper, Jia \textit{et al.}~\cite{jia2019efficient} developed an exact, efficient algorithm to compute the Shapley value for 
$K$NN classifiers. In principle,
we can use a $K$NN classifier to 
act as a surrogate model and use it
instead of the target learning algorithm. 
Given a single validation point $x_\text{val}$ with the label $y_\text{val}$, the simplest, unweighted version of a $K$NN classifier first finds the top-$K$ training points $(x_{\alpha_1},\cdots,x_{\alpha_K})$ that are most similar to $x_\text{val}$ and outputs the probability of $x_\text{val}$ taking the label $y_\text{val}$ as $P[x_\text{val} \rightarrow y_\text{val}]=\frac{1}{K}\sum_{i = 1}^{K} \mathbbm{1}[y_{\alpha_i} = y_\text{val}]$. 
We assume that the confidence of predicting the right label is used as the performance measure, i.e., 
\begin{align}
\begin{small}
\label{eqn:utility_classification_unweighted}
  U(S) =\frac{1}{K} \sum_{k=1}^{\min\{K,|S|\}} \mathbbm{1}[y_{\alpha_k(S)} = y_\text{val}]
\end{small}
\end{align}
where $\alpha_k(S)$ represents the index of the training feature that is the $k$th closest to $x_\text{val}$ among the training examples in $S$. Particularly, $\alpha_k(D)$ is abbreviated to $\alpha_k$. Under this performance measure, the Shapley value can be calculated exactly using the following theorem.

\begin{Thm}[Jia \textit{et al.}~\cite{jia2019efficient}]
\label{thm:knn_shapley}
Consider the model performance measure in (\ref{eqn:utility_classification_unweighted}). Then, the Shapley value of each training point can be calculated recursively as follows:
\begin{small}
\begin{align}
\label{eqn:KNN_unweighted_class_1}
&\nu(z_{\alpha_N})=\frac{\mathbbm{1}[y_{\alpha_{N}} = y_\text{val}]}{N}
\\
\label{eqn:KNN_unweighted_class_2}
& \nu(z_{\alpha_i}) = \nu(z_{\alpha_{i+1}})\!\! +  \frac{\mathbbm{1}[y_{\alpha_i} = y_\text{val}] - \mathbbm{1}[y_{\alpha_{i+1}} = y_\text{val}]}{K}
    \frac{\min\{K, i\}}{i}
\end{align}
\end{small}
\end{Thm}
%\vspace{-3mm}
Theorem~\ref{thm:knn_shapley} can be readily extended to the case of multiple validation points by summing up the Shaplley value with respect to each validation point.
We will call the scores obtained from (\ref{eqn:KNN_unweighted_class_1}) and (\ref{eqn:KNN_unweighted_class_2}) \textit{the $K$NN-Shapley value} hereinafter. For each validation point, computing \textit{the $K$NN-Shapley value} requires only $\OO(N\log N)$ time, which circumvents the exponentially many utility evaluations entailed by the Shapley value definition.  

\textbf{Using Pre-trained Embeddings.} One  problem of using $K$NN as a surrogate model is that 
$K$NN often does not perform well on high-dimensional data. 
As many works have
illustrated the power of pre-trained embeddings on a different, new task~\cite{kornblith2019better,kolesnikov2019revisiting,qiu2020pre,zhai2019large},
we address this problem by using \textit{pre-trained embeddings} as a feature extractor and
then apply $K$NN.
Note that this feature
needs to be trained 
on a \textit{different dataset} for $K$NN surrogate to respect Shapley value semantics.

\section{Theoretical Comparison Between LOO and Shapley Value}
\label{sec:theory}
%\vspace{-0.5em}

We now focus on the first question: \textit{What’s the relationships and differences, both theoretically and empirically, between these two lines of approaches?} 
Specifically, we define two 
performance metrics and conduct
theoretical analysis \textit{under
different technical assumptions}.
To our best knowledge,
this is the first theoretical analysis
reasoning about the relative performances of different techniques that measure data importance.

%\vspace{-3mm}
%\subsection{Predictive Power of the Value Measures}
\subsection{Performance Metric 1: Order-Preservation}%Predictive Power of the Value Measures}
%\vspace{-2mm}
% In practice, the valuation methods often serve as a preprocessing step to filter out low-quality data, such as mislabeled or noisy data, in a given dataset.
%\revise{To justify that the data importance produced by a valuation technique can reflect the data usefulness in practice, existing valuation techniques are often examined in terms of their performance to work as a pre-processing step to filter out low-quality data, such as mislabeled or noisy data, in a given dataset.}
%Then, one may train a model based only on the remaining ``good'' data instances or their combination with additional data. Note that 
Both the LOO-based method and the Shapley-based method only measure the importance of a data point relative to other points in the given dataset. Since it is still uncertain what data will be used in tandem with the point being valued after its importance is measured, \textit{in the 
first performance metric}, we hope that the data importance measures of a point are indicative of the expected performance boost when combining the point with a random set of data points.

In particular, we consider two points that have different scores under a given data importance measure and study whether the expected model performance improvements due to the addition of these two points will have the same order as the importance scores. With the same order, we can confidently select the higher-importance point in favor of another when performing ML tasks. We formalize this desirable property in the following definition.

\begin{Df}
We say a data importance measure $\nu$ is order-preserving at a pair of training points $z_i,z_j$ with different scores if $\big(\nu(z_i,U) - \nu(z_j,U)\big)\times\E\big[U(T\cup \{z_i\}) - U(T\cup \{z_j\})\big] > 0$
% \begin{align}
%     \big(\nu(z_i,U) - \nu(z_j,U)\big)\times\E\big[U(T\cup \{z_i\}) - U(T\cup \{z_j\})\big] > 0
% \end{align}
where $T$ is an arbitrary random set drawn from some distribution.
\end{Df}

For general model performance measures $U$, it is difficult to analyze the order-preservation of the corresponding data importance measures. However, for $K$NN, we can precisely characterize this property for both the LOO and the Shapley value. The formula for \textit{the $K$NN-Shapley value} is given in Theorem~\ref{thm:knn_shapley} and we present the expression for \textit{the $K$NN-LOO value} in the following lemma.

\begin{Lm}[$K$NN-LOO Value]
\label{lm:knn_loo}
Consider the model performance measure in (\ref{eqn:utility_classification_unweighted}). Then, \textit{the $K$NN-LOO value} of each training point can be calculated by $\nu_\text{loo}(z_{\alpha_i}) = \frac{1}{K}\big(\mathbbm{1}[y_{\alpha_i} = y_\text{val}] - \mathbbm{1}[y_{\alpha_{K+1}} = y_\text{val}] \big)$ if $i\leq K$ and $0$ otherwise.
% \begin{small}
% \begin{align}
%     \nu_\text{loo}(z_{\alpha_i}) = \left\{\begin{array}{cc}
%          \frac{1}{K}\big(\mathbbm{1}[y_{\alpha_i} = y_\text{val}] - \mathbbm{1}[y_{\alpha_{K+1}} = y_\text{val}] \big)& \text{if $i\leq K$}\\
%          0 & \text{if $i \geq K+1$}
%     \end{array}
%     \right.
% \end{align}
% \end{small}
\end{Lm}

Now, we are ready to state the theorem that exhibits the order-preservation of \textit{the $K$NN-LOO value} and \textit{the $K$NN-Shapley value}.

\begin{restatable}{Thm}{goldbach}
\label{thm:comparison}
For any given $D=\{z_1,\ldots,z_N\}$, where $z_i = (x_i,y_i)$, and any given validation point $z_\text{val}=(x_\text{val},y_\text{val})$, assume that $z_1,\ldots,z_N$ are sorted according to their similarity to $x_\text{val}$. Let $d(\cdot,\cdot)$ be the feature distance metric according to which $D$ is sorted. Suppose that $P_{(X,Y)
\in\D}(d(X,x_\text{val})\geq d(x_i,x_\text{val}))>\delta$ for all $i=1,\ldots,N$ and some $\delta > 0$. Then, $\nu_\text{shap-knn}$ is order-preserving for all pairs of points in $I$; $\nu_\text{LOO-knn}$ is order-preserving only for $(z_i,z_j)$ such that $\max{i,j}\leq K$.
\end{restatable}

Due to the space limit, we will omit all proofs to the appendix. The assumption that $P_{(X,Y)
\in\D}(d(X,x_\text{val})\geq d(x_i,x_\text{val}))>\delta$ in Theorem~\ref{thm:comparison} intuitively means that it is possible to sample points that are further away from $x_\text{val}$ than the points in $D$. This assumption can easily hold for reasonable data distributions in continuous space.

Theorem~\ref{thm:comparison} indicates that \textit{the $K$NN-Shapley value} has more predictive power than \textit{the $K$NN-LOO value}---\textit{the $K$NN-Shapley value} can predict the relative utility of any two points in $D$, while \textit{the $K$NN-LOO value} is only able to correctly predict the relative utility of the $K$-nearest neighbors of $x_\text{val}$. In Theorem~\ref{thm:comparison}, the relative data utility of two points is measured in terms of the model performance difference when using them in combination with a random dataset. 

Theorem~\ref{thm:comparison} can also be generalized to the setting of multiple validation points using the additivity property. Specifically, for any two training points, \textit{the $K$NN-Shapley value} with respect to multiple validation points is order-preserving when the order remains the same on each validation point, while \textit{the $K$NN-LOO value} with respect to multiple validation points is order-preserving when the two points are within the $K$-nearest neighbors of all validation points and the order remains the same on each validation point. We can see that similar to the single-validation-point setting, the condition for \textit{the $K$NN-LOO value} with respect to multiple validation points to be order-preserving is more stringent than that for \textit{the $K$NN-Shapley value}. 

%Moreover, although in the definition, order-preservation is proposed as a performance metric for data importance measures, it can also be used as a performance metric for an estimator of data importance. A data importance estimator will be order-preserving if the estimation error is much smaller than the gap between the data importance measures of every two points in the training set. Since the estimation error of a consistent estimator can be made arbitrarily small given enough samples, a consistent estimator (if exists) for an order-preserving data importance measure is also order-preserving when the sample size is large. An example of such estimator for the Shapley value is the sample average of the marginal contribution of a point to the ones preceding it over multiple random permutations. 

%\vspace{-2mm}
%\subsection{Usability for Differentially Private Algorithms}
\subsection{Performance Metric 2: Value Distinguishness}
%\vspace{-2mm}

In the second performance metric, we 
are interested in conditions under which
LOO-based method and Shapley-based method
cannot distinguish between different
data points, independently of their 
importance. The technical tool that 
we use as a demonstration is to consider the setting in which
the classifier is trained 
in a differentially private (DP) manner.

%Since the datasets used for machine learning tasks often contain sensitive information (e.g., medical records), it has been increasingly prevalent to develop privacy-preserving learning algorithms. Hence, it is also interesting to study how to value data when the learning algorithm preserves some notion of privacy. Differential privacy (DP) has emerged as a strong privacy guarantee for algorithms on aggregate datasets. The idea of DP is to carefully randomize the algorithm so that the output does not depend too much on any individuals' data.

\begin{Df}[Differential privacy]
$\mathcal{A}:\mathcal{D}^N\rightarrow \mathcal{H}$ is $(\epsilon,\delta)$-differentially private if for all $R\subseteq \mathcal{H}$ and for all $D,D'\in \mathcal{D}^N$ such that $D$ and $D'$ differ only in one data instance: $P[\mathcal{A}(D)\in R]\leq e^\epsilon P[\mathcal{A}(D')\in R] + \delta$.
% \begin{align}
%     P[\mathcal{A}(D)\in R]\leq e^\epsilon P[\mathcal{A}(D')\in R] + \delta
% \end{align}
\end{Df}

By definition, differentially private learning algorithms hide the influence of one training point on the model performance. Thus, it may be more difficult to differentiate ``good'' data from ``bad'' ones for differentially private models. We will show that the Shapley value could have more discriminative power than the LOO value when the learning algorithms satisfy DP. 

% To formally characterize the discriminative power of the value measures, we consider some ``completely useless'' training points and use their values as a baseline.

% \begin{Df}
% We say a training point $z_i\in D$ is dummy if $U(S\cup \{z_i\}) = U(S)$ for all $S\subseteq D$. 
% \end{Df}

% It is easy to verify that the LOO value and the Shapley value for dummy data are both zero. 

The following theorem states that for differentially private algorithms, the values of training data are gradually indistinguishable from each other as the training size grows larger using both the LOO and the Shapley value measures; nonetheless, the value differences vanish faster for the LOO value than the Shapley value.

\begin{restatable}{Thm}{differentialprivacy}
\label{thm:dp_value}
For a learning algorithm $\A(\cdot)$ that achieves $(\epsilon(N),\delta(N))$-DP when training on $N$ data points, let the performance measure be $U(S) = -\frac{1}{M}\sum_{i=1}^M \E_{h\sim \A(S)}l(h,z_{\text{val},i})$ for $S\subseteq D$. 
% Assume that there exists a dummy point $z^*$ in $D$. 
Let $\epsilon'(N) = e^{c\epsilon(N)} - 1 + ce^{c\epsilon(N)}\delta(N)$. It holds that
% % \begin{small}
% % \begin{align}
% %     &\max_{z_i\in D} \nu_\text{loo}(z_i) - \nu_\text{loo}(z^*) \leq \epsilon'(N-1)\\
% %     &\max_{z_i\in D} \nu_\text{shap}(z_i) - \nu_\text{shap}(z^*) \leq \frac{1}{N-1}\sum_{i=1}^{N-1}\epsilon'(i)
% % \end{align}
% \end{small}
\begin{small}
\begin{align*}
    &\max_{z_i\in D} \nu_\text{loo}(z_i) \leq \epsilon'(N-1)
    &\max_{z_i\in D} \nu_\text{shap}(z_i) \leq \frac{1}{N-1}\sum_{i=1}^{N-1}\epsilon'(i).
\end{align*}
\end{small}
\end{restatable}
%\vspace{-2mm}
For typical differentially private learning algorithms, such as adding random noise to stochastic gradient descent, the privacy guarantees will be weaker if we reduce the size of training set (e.g., see Theorem 1 in Abadi \textit{et al.}~\cite{abadi2016deep}). In other words, $\epsilon(n)$ and $\delta(n)$ are monotonically decreasing functions of $n$, and so is $\epsilon'(n)$. Therefore, it holds that $\epsilon'(N)< \frac{1}{N}\sum_{i=1}^N \epsilon'(i)$. The implications of Theorem~\ref{thm:dp_value} are three-fold. Firstly, the fact that the maximum score of all training points is directly upper bounded by $\epsilon'$ signifies that stronger privacy guarantees will naturally increase the difficulty to distinguish the importance of different points. Secondly, the monotonic dependency of $\epsilon'$ on $N$ indicates that both the LOO and the Shapley value converge to zero when the training size is very large. Thirdly, by comparing the upper bound for the LOO and the Shapley value, we see that the convergence rate of the Shapley value is slower and thus it has a better chance to differentiate ``good'' data from the ``bad'' ones compared with the LOO value.
% This is also corroborated by our experimental findings in Section~\ref{section:exp}.

Our results are extendable to general \textit{stable} learning algorithms, which are insensitive to the removal of an arbitrary point in the training dataset~\cite{bousquet2002stability}. Stable learning algorithms are appealing as they enjoy provable generalization error bounds. Indeed, differentially private algorithms are subsumed by the class of stable algorithms~\cite{Wang2015-pe}. 
%A broad variety of other learning algorithms are also stable, including all learning algorithms with Tikhonov regularization. 
We leave the details to the appendix.

%\vspace{-0.5em}
\section{Empirical Studies}
\label{section:exp}

% In this section, we first compare the runtime of our algorithm with the existing works on various dataset. Then, we compare the usefulness of the data importance produced by different algorithms based on various applications, including mislabeled data detection, watermark removal, data summarization, active data acquisition, and domain adaptation. We will leave the detailed experimental setting, such as model architecture and hyperparamters of training processes, to the appendix.

Here we conduct a thorough empirical study on a range of real-world ML applications with different datasets to investigate the performance comparison between (a) leave-one-out--based method and Shapley-based method, and (b) exact Shapley-based method and Shapley over $K$NN surrogates.
We first compare the runtime for different data importance quantification methods, followed by the data importance predictive power comparison, which is demonstrated on applications including mislabeled data detection, watermark removal, data summarization, active data acquisition, and domain adaptation.
Due to the space limit, we leave the detailed experimental settings to appendix.

% \textcolor{red}{CZ: Rephrase this section
% such that it reads like a study instead of
% purely selling KNN-Shapley}

% \textcolor{red}{CZ: Add a summary of results
% to be consisten with Section 1's contribution list.}

% %\vspace{-0.5em}
\subsection{Data Importance Quantification Approaches}
\label{sec:approaches}
% %\vspace{-0.5em}

Here we we mainly compare the up-to-date data importance quantification approaches, including the exact Shapley-based method, leave-one-out method, as well as the ones using $K$NN as surrogates for both.
% \vspace{-1mm}

\textbf{Truncated Monte Carlo Shapley (TMC-Shapley).} This is a Monte Carlo-based approximation algorithm proposed in Ghorbani \textit{et al.}~\cite{ghorbani2019data}. Monte Carlo-based methods
% The central idea behind Monte Carlo-based approximation algorithms is to 
regard the Shapley value as the expectation of a training instance's marginal contribution to a random set and then use the sample mean to approximate it. 
% Evaluating the marginal contribution to a different set requires to retrain the model, which bottlenecks the efficiency of Monte Carlo-based methods. TMC-Shapley combined the Monte Carlo method with a heuristic that ignores the random sets of large sizes since the contribution of a data point to those sets will be small.\\
%

\textbf{Gradient Shapley (G-Shapley).} This is another Monte Carlo-based method proposed in Ghorbani \textit{et al.}~\cite{ghorbani2019data} with a different heuristic to accelerate the algorithm. G-Shapley approximates the model performance change due to the addition of one training point by taking a gradient descent step at that point and calculating the performance difference. This method is applicable only to the models trained with gradient methods; hence, the method will be included as a baseline in our experimental results when the underlying models are trained using gradient methods.

\textbf{Leave-One-Out (LOO).} We use LOO to refer to the algorithm that calculates the exact model performance due to the removal of a training point. Evaluating the LOO error requires to re-train the model on the reduced dataset for every training point, thus also impractical for large models.

\textbf{$K$NN-LOO.} Leave-one-out is efficient for $K$NN according to Theorem~\ref{lm:knn_loo}. To use the $K$NN-LOO for valuing data, we first use the pre-trained models offered in PyTorch~\cite{paszke2017pytorch} to extract features for $K$NN and compute the $K$NN-LOO value over the extracted features.

\textbf{$K$NN-Shapley.} We use $K$NN-Shapley
to refer to the following algorithm:
similar to in $K$NN-LOO, we first use the pretrained models in PyTorch to extract features.
% apply a feature embedding pre-trained on a \textit{different} dataset, available from TensorFlow Hub to extract features for $K$NN.
We then directly 
apply Theorem~\ref{thm:knn_shapley} to 
compute the Shapley value over pre-trained
feature transformations. 
When pre-trained feature
transformations are not available, we directly compute \textit{the $K$NN-Shapley value} on the raw data as a surrogate for the true Shapley value. 
The complexity of the above algorithm is $\OO(Nd + N\log N)$ where $d$ is the dimension of  feature representation. As opposed to Monte Carlo-based methods (e.g., \cite{ghorbani2019data,jia2019towards}), the proposed algorithm does not require retraining models. It is well suited for approximating scores for large models. 
 
\textbf{Random.} The random baseline does not differentiate importance between different data points and selects data randomly from training set to perform a given task.

\subsection{Runtime Comparison}

\begin{figure}[t!]
\begin{center}
\centering
\includegraphics[trim=0 0 0 2.2cm,clip,width=0.7\columnwidth]{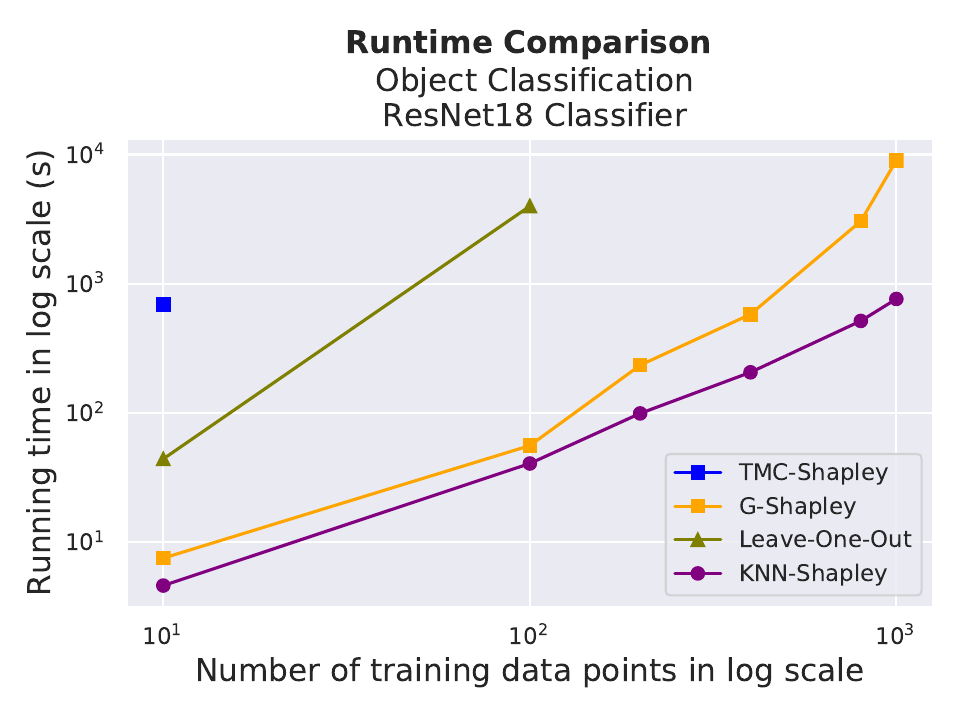}
\caption{\small runtime comparison.}\label{fig:runtime}
\end{center}
\end{figure}

First, we compare the runtime between the $K$NN-Shapley approach and other baselines. Fig.~\ref{fig:runtime} corresponds to ResNet-18~\cite{he2016deep} on CIFAR-10~\cite{cifar10}, implemented on a machine with 1.80 GHz and 32 GB memory.
We can see that $K$NN-Shapley (using pre-trained MobileNet~\cite{howard2017mobilenets} embedding) outperforms other approaches by several orders of magnitude for large training data size and large model size. 
% This demonstrates the utility of $K$NN-Shapley, given that computing the Shapley value is one of the main challenges currently.

% \vspace{-2mm}
\subsection{Comparisons on Applications}
\label{subsec:comp}

\begin{table}[]
    \centering
    \scriptsize
    \caption{\small a summary of experiments in Appendix~\ref{appendix:sec_exp}.}
    \label{tab:appendix-exp}
\begin{tabular}{L{4cm}L{3.5cm}}
\toprule
{\bf Task (Section)} &  {\bf Datasets} \\
\midrule
1. Noisy labels Detection (\ref{appendix:sec_detect}) & Spam~\cite{tiago2013spam}, Flower\footnoteref{note:flower} \\
% \multirow{2}{*}{Pattern-based watermark removal (\ref{appendix:sec_removal})}
%  & Fashion-MNIST~\cite{xiao2017fashion},\\
%  & MNIST~\cite{lecun2010mnist}, PubFig-83~\cite{kumar2009pubfig} \\
2. Pattern-based watermark removal (\ref{appendix:sec_removal}) & Fashion-MNIST~\cite{xiao2017fashion}, MNIST~\cite{lecun2010mnist}, PubFig-83~\cite{kumar2009pubfig} \\
3. Instance-based watermark removal (\ref{appendix:sec_detect}) & CIFAR-10~\cite{cifar10}, SVHN~\cite{Netzer2011ReadingDI} \\
4. Data summarization (\ref{appendix:sec_ds}) & Tiny ImageNet~\cite{Le2015TinyIV} \\
5. Data acquisition (\ref{appendix:sec_da}) & Tiny ImageNet~\cite{Le2015TinyIV} \\
6. Domain adaptation (\ref{appendix:sec_adaptation}) & MNIST~\cite{lecun2010mnist} $\rightarrow$ SVHN~\cite{Netzer2011ReadingDI} \\
\bottomrule
\end{tabular}
\end{table}

% %\vspace{-0.5em}
%\revise{Most of the following applications are discussed in recent work on data valuation~\cite{ghorbani2019data}.  In this paper, we hope to understand: \textit{Can a simple, scalable heuristic to approximate the Shapley value using a KNN surrogate outperform these, often more computationally expensive, previous approaches on the same set of applications?} As a result, our goal is not to outperform state-of-the-art methods for each application; instead, we hope to put our work in the context of current efforts in understanding the relationships between different data valuation techniques and their performance on these tasks.}

% We use the following datasets in our empirical study. 
We study the efficacy of data importance estimated by different approaches on a range tasks, including noisy label detection, watermark removal, data summarization, active data acquisition, and domain adaptation. While most of the applications are used in a recent work~\cite{ghorbani2019data},
the watermark removal including both pattern-based and instance-based watermark removal evaluations are proposed by us here.
We consistently use the 
MobileNet embedding
pretrained on ImageNet
on image inputs.
For the spam
dataset~\cite{tiago2013spam} and the 
flower dataset\footnote{\label{note:flower}\footnotesize\url{https://www.tensorflow.org/tutorials/load\_data/images}},
we do not apply embedding
as the provided features
are not of image form.
We compare
different embeddings in Appendix~\ref{appendix:sec_embedding}. 

%
%in this section each correspond to one single embedding; more specifically, in the task of detecting noisy labels, we use raw features for the spam dataset, the embedding extracted by InceptionV3 for the flowerdataset, and the embedding extracted by MobileNet for all the other remaining datasets and applications.

\begin{figure}[H]
\begin{center}
\begin{small}
\begin{subfigure}{0.48\linewidth}
\centering
\includegraphics[width=\columnwidth]{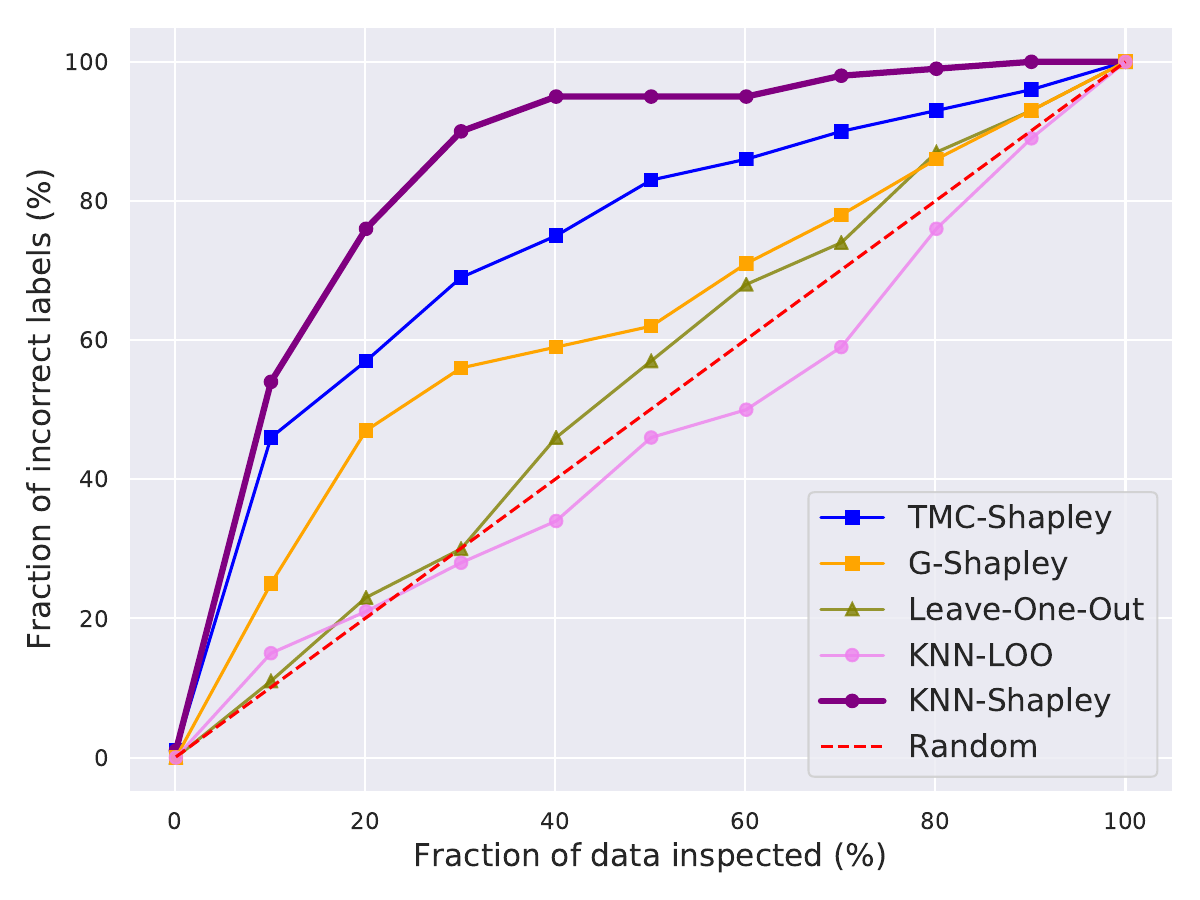}
\caption{\small Noisy labels detection}\label{fig:exp-1}
\end{subfigure}
\begin{subfigure}{0.48\linewidth}
\centering
\includegraphics[width=\columnwidth]{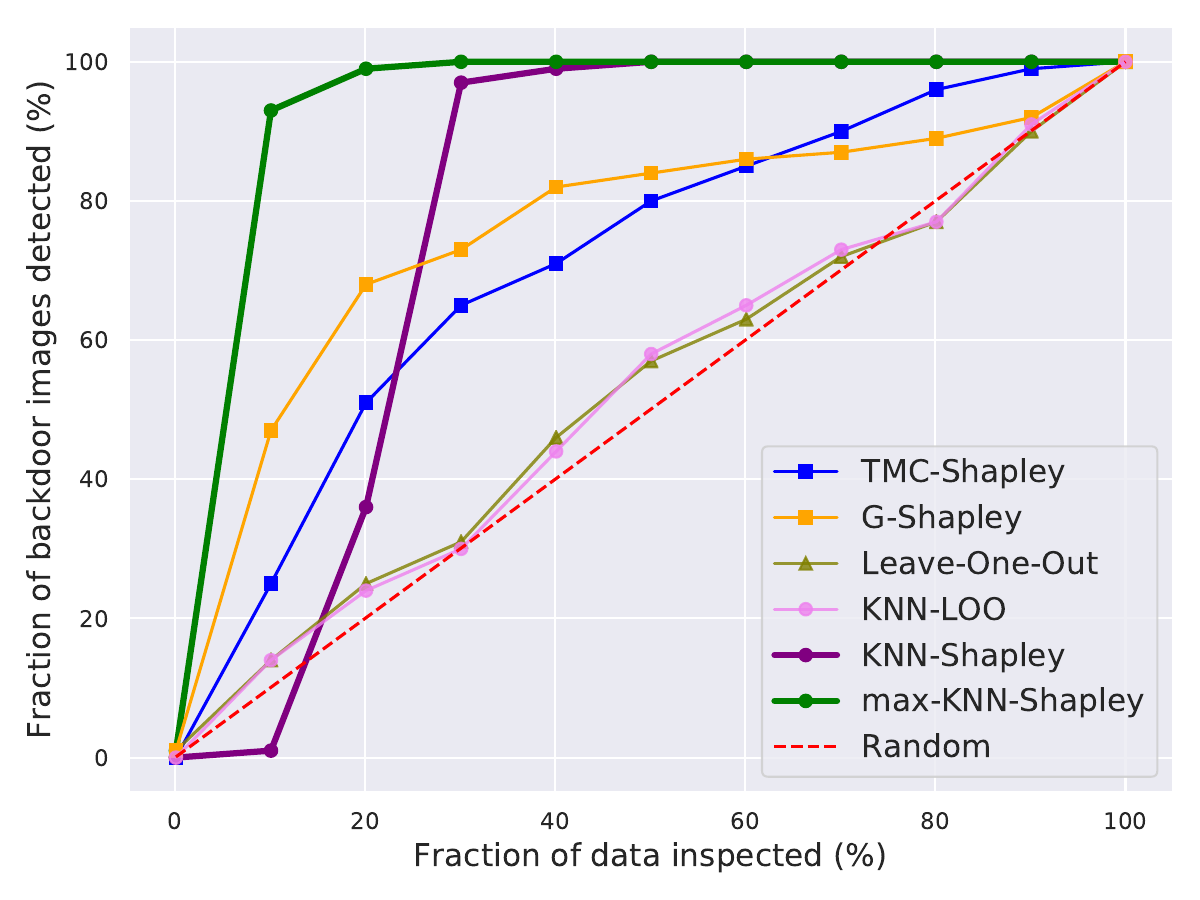}
\caption{\small Watermark removal}\label{fig:exp-2}
\end{subfigure}9
\begin{subfigure}{0.48\linewidth}
\centering
\includegraphics[width=\columnwidth]{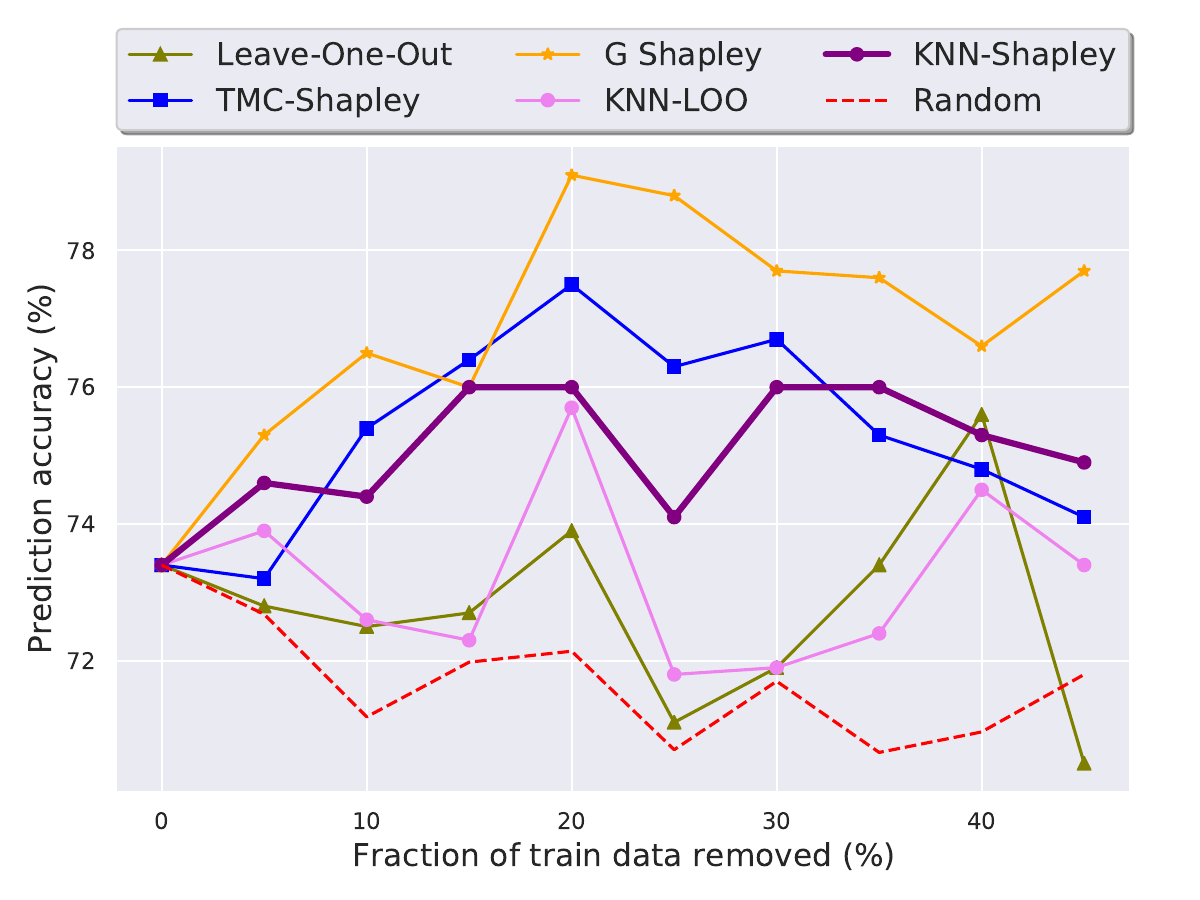}
\caption{\small Data summarization}\label{fig:exp-3}
\end{subfigure}
\begin{subfigure}{0.48\linewidth}
\centering
\includegraphics[width=\columnwidth]{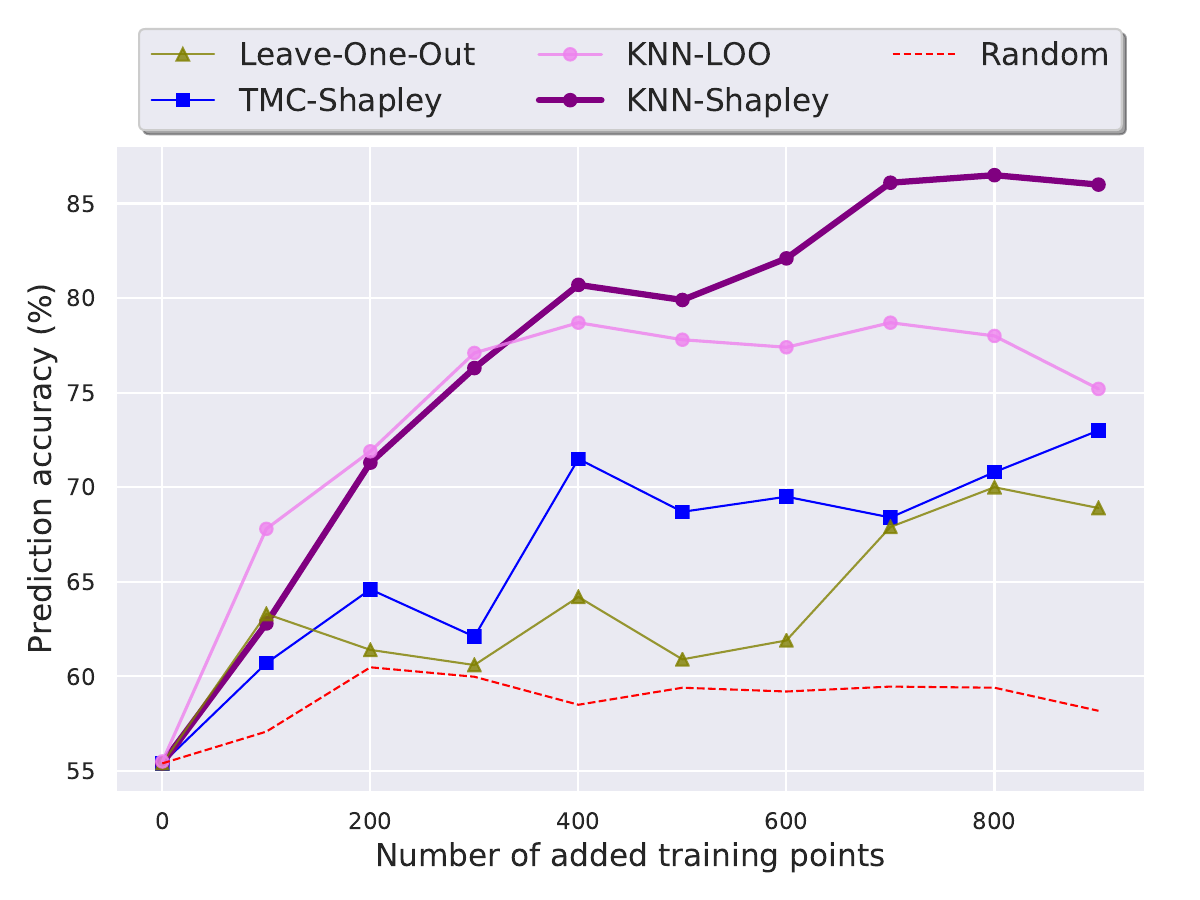}
\caption{\small Data acquisition}\label{fig:exp-4}
\end{subfigure}
\end{small}
\end{center}
% \vspace{-1em}
\caption{\small The experiment result of (a) noisy label detection on fashion-MNIST dataset; (b) instance-based watermark removal on MNIST dataset; (c) data summarization on UCI Adult Census dataset~\cite{ron1996uci}; (d) data acquisition on MNIST dataset with injected noise.  In (a)-(b) the ``random'' line shows the results of random guess; while in (c)-(d), the ``random'' line corresponds to the empirical results of the random baseline introduced in Section~\ref{sec:approaches}.}\label{fig:exp_overall}
% \vspace{-2em}
\end{figure}

% wrapfigure_of_summary
% \begin{wrapfigure}{r}{0.7\textwidth}
% %\vspace{-3em}
% \begin{minipage}{0.7\textwidth}
% \begin{table}[H]
% \centering
% \caption{A summary of experiments in Appendix~\ref{appendix:sec_exp}}
% \label{tab:summary}
% \scriptsize
% \begin{tabular}{lll}
% \toprule
% {\bf Task (Section)} &  {\bf Datasets} & {\bf \makecell{}}\\
% \midrule
% Noisy labels Detection (\ref{appendix:sec_detect}) & Spam, Flower \\
% Pattern-based watermark removal (\ref{appendix:sec_removal}) & Fashion-MNIST~\cite{xiao2017fashion}, MNIST~\cite{lecun2010mnist}, PubFig-83~\cite{kumar2009pubfig} \\
% Instance-based watermark removal (\ref{appendix:sec_detect}) & CIFAR-10~\cite{cifar10}, SVHN~\cite{Netzer2011ReadingDI} \\
% Data summarization (\ref{appendix:sec_ds}) & Tiny ImageNet~\cite{Le2015TinyIV} \\
% Data acquisition (\ref{appendix:sec_da}) & Tiny ImageNet~\cite{Le2015TinyIV} \\
% Domain adaptation (\ref{appendix:sec_adaptation}) & MNIST~\cite{lecun2010mnist} $\rightarrow$ SVHN~\cite{Netzer2011ReadingDI} \\
% \bottomrule
% \end{tabular}
% %\vspace{-5mm}
% \end{table}
% \end{minipage}
% \end{wrapfigure}

%\vspace{-1em}
\paragraph{Summary of Experiments in Appendix.} In this paper, we evaluate different
methods on 6 applications, each of which on 2 to 3 datasets following the practice of previous work.
Due to the space limitation, we describe only \textit{one} representative dataset in the main body and leave the experiment details and results on other datasets to Appendix~\ref{appendix:sec_exp}. Table~\ref{tab:appendix-exp} is a summary of the experiments in the appendix.

% \begin{wrapfigure}{r}{0.3\textwidth}
% %\vspace{-10mm}
% \begin{center}
% \includegraphics[width=0.3\textwidth]{figs/Label_c.pdf}
% \end{center}
% %\vspace{-5mm}
% \end{wrapfigure}

%\vspace{-0.5em}

\begin{figure*}
\begin{subfigure}{0.245\textwidth}
\centering
\includegraphics[width=\columnwidth]{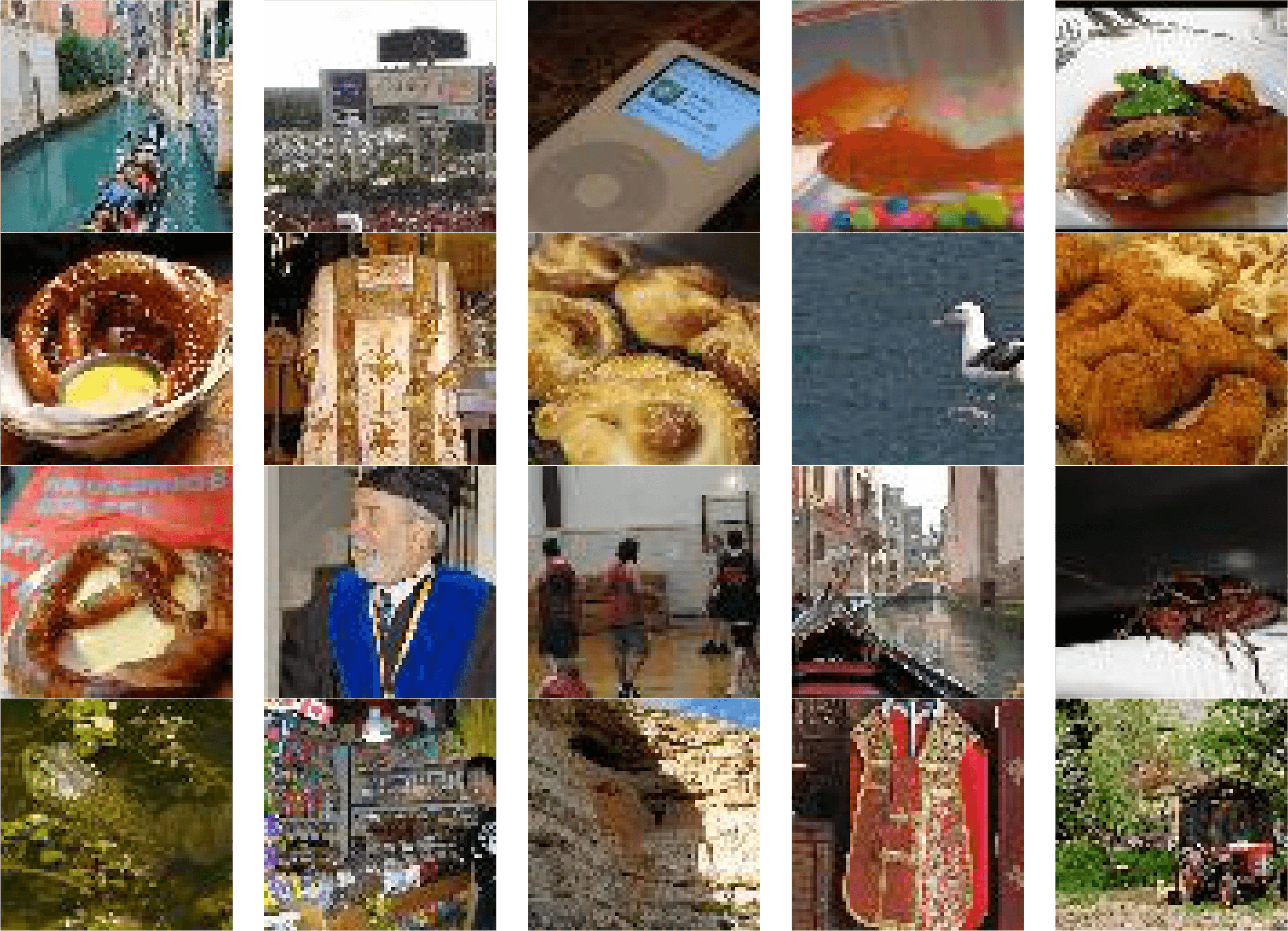}
\caption{}\label{fig:ds-mobilenet}
\end{subfigure}
\begin{subfigure}{0.245\textwidth}
\centering
\includegraphics[width=\columnwidth]{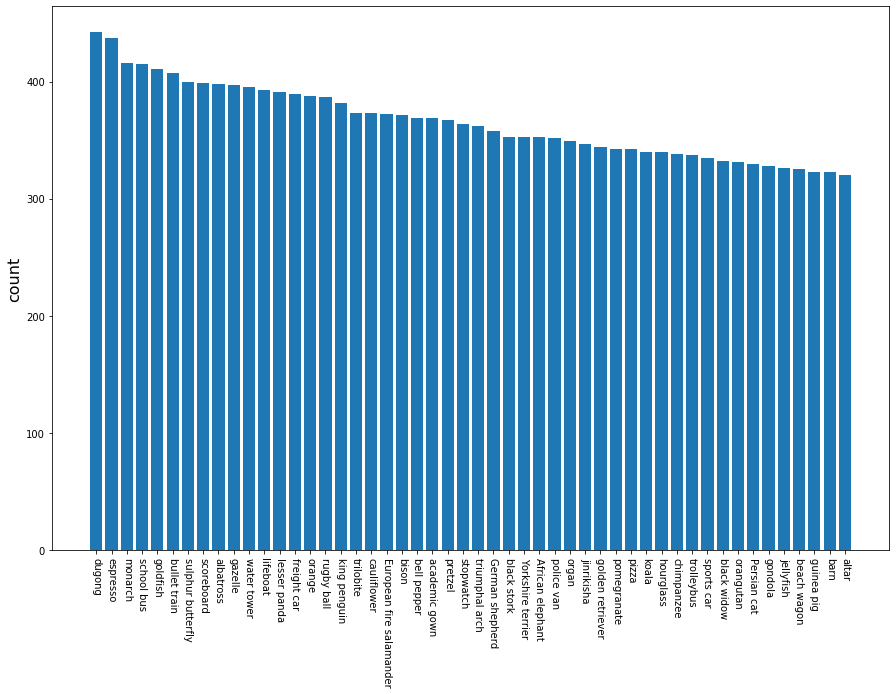}
\caption{}\label{fig:ds-bar-mobilenet}
\end{subfigure}
\begin{subfigure}{0.245\textwidth}
\centering
\includegraphics[width=\columnwidth]{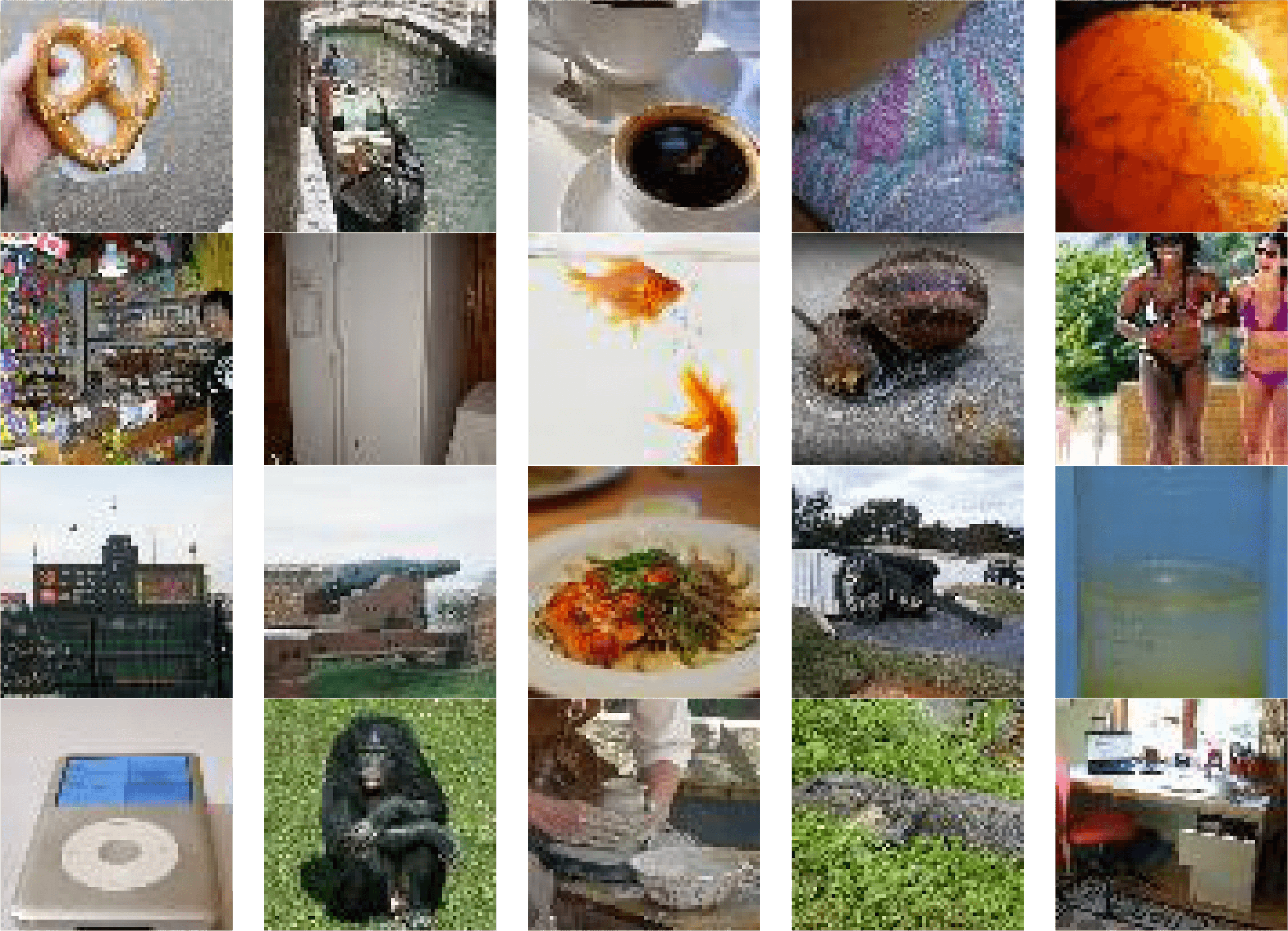}
\caption{}\label{fig:ds-vgg}
\end{subfigure}
\hfill
\begin{subfigure}{0.245\textwidth}
\centering
\includegraphics[width=\columnwidth]{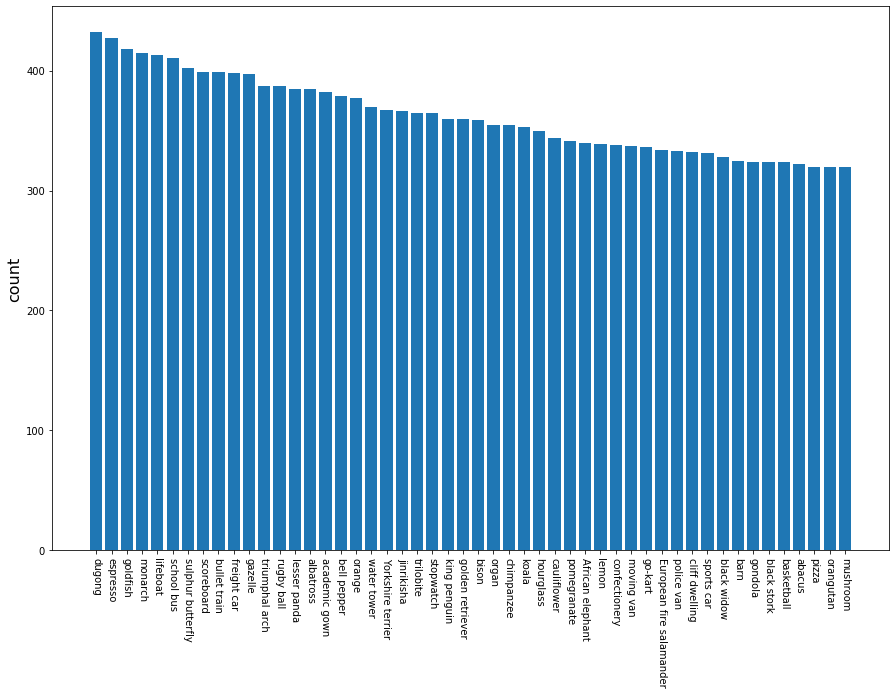}
\caption{}\label{fig:ds-bar-vgg}
\end{subfigure}
% \vspace{-1em}
\caption{\small (a) (c) Top 20 selected images with the highest Shapley values for Tiny ImageNet, using MobileNet and VGG11 embeddings.
% (b) (d) statistics for the class distribution of the top 50 classes with the most images chosen. 
(b) (d) Histogram of images in top 50 classes after summarization (sorted by  decreasing order).
It is clear that within the top 50 classes, there are many overlapped classes between different embeddings. 
}\label{fig:ds-viz}
\end{figure*}

\paragraph{Noisy Labels Detection.}
\phantomsection
\label{sec:detect}
Labels in the real world are often noisy due to automatic labeling, non-expert labeling, or label corruption by data poisoning adversaries. We show that the notion of data importance can help prioritize the verification process, allowing experts to review only the examples that are most likely to be contaminated. \textit{The key idea is to rank the data points according to their data importance and prioritize the points with the lowest importance scores. }
Following Ghorbani \textit{et al.}~\cite{ghorbani2019data}, we perform experiments in three settings
and present the result of a three-layer convolutional network trained on the fashion-MNIST dataset here in the main body. The noise flipping ratio is 10\% for this dataset.
% : a Naive Bayes model trained on a spam classification dataset, a logistic regression model trained on Inception-V3 features of a flower classification dataset, and a three-layer convolutional network trained on fashion-MNIST dataset. 
% The noise flipping ratio is 20\%, 10\%, and 10\%, respectively. 
The performance of different data importance measures is illustrated in Fig.~\ref{fig:exp-1}. We examine the label of the training instances that have the lowest scores, and plot the change of the fraction of detected mislabeled data (in percentage) with the fraction of the checked training data (in percentage). We can see that \textit{the $K$NN-Shapley value} outperforms all other methods. Also, the Shapley value--based measures, including TMC-Shapley, G-Shapley, and our KNN-Shapley, are more effective than the LOO-based measures.

%We present the results on a spam dataset and a flower dataset in Appendix~\ref{appendix:sec_detect}.

% Xuehui
% In \cite{ghorbani2019data}, an experiment was presented to verify how SV can reflect the quality of data, during which three different data sets were trained on three different predictive models, and some portions of training data were mislabeled artificially. The result in \ref{fig:label} showed that if we inspect the data points according to SV from \textit{the least valuable} to \textit{the most valuable}, mislabeled data points tend to be among the ones with least values and can be discovered early on. We re-implemented their experiment by their TMC-Shapley and G-Shapley (only available for logistic regression and convolutional neural network models) methods, as well as Leave-One-Out method and random choosing, which were both used as baseline algorithms by them, and also tried $K$NN-Shapley method upon the same models and the same data sets with their experiment.

% \vspace{-1.5em}
\paragraph{Watermark Removal.}
\phantomsection
\label{sec:watermark}

% In pattern-based watermark removal, we adopted two types of patterns: one is to change the pixel values at the corner of an image~\cite{chen2018detecting}, another is to blend a specific word (like ``TEST'') into an image, as shown in Fig.~\ref{fig:watermark2}. Specifically, after an image is blended with the ``TEST'' pattern, there is high chance that it is classified as the target label, e.g, an ``automobile'' on CIFAR-10.
% The first pattern is used in the experiments on fashion MNIST and MNIST, which is composed of single channel images. The second pattern is applied to Pubfig-83 which contains multi-channel images.

% In instance-based watermark removal, we used the same watermarks as~\cite{adi2018turning}, which contains a set of abstract images with specific assigned labels. The example of a trigger image is shown in Fig.~\ref{fig:watermark3}. This type of watermarks are typically chosen from out-of-distribution data.

% Despite their tremendous success in various fields, training neural networks from scratch could be computationally expensive and requires a lot of training data.
% To solve the problem, one may contribute a dataset and outsource the computation to a different party. 
One prevalent way to claim the ownership of a trained deep net is to embed watermarks into the model.
% But how can the dataset contributor claim the ownership of the data source of a trained model? A prevalent way of addressing this question is to . 
There are two classes of watermarking techniques, namely,
% in the existing work~\cite{chen2019refit},
pattern-based techniques and instance-based techniques. The watermark examples are displayed in Fig.~\ref{fig:watermarks}.
Here, we present the experiment results for instance-based techniques. The details on how they the watermarks are generated and how they work, as well as the experiment results for the pattern-based techniques are left to Appendix~\ref{appendix:sec_removal}.

\begin{figure}[htbp]
\begin{center}
\begin{subfigure}{0.155\textwidth}
\centering
\includegraphics[width=\columnwidth]{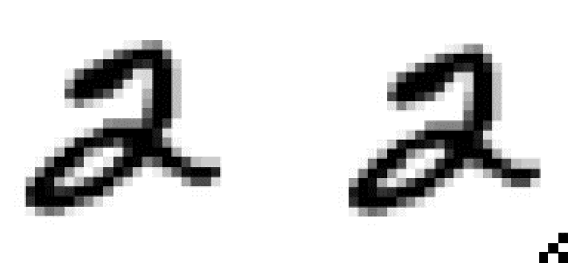}
\caption{}\label{fig:watermark1}
\end{subfigure}
\begin{subfigure}{0.155\textwidth}
\centering
\includegraphics[width=\columnwidth]{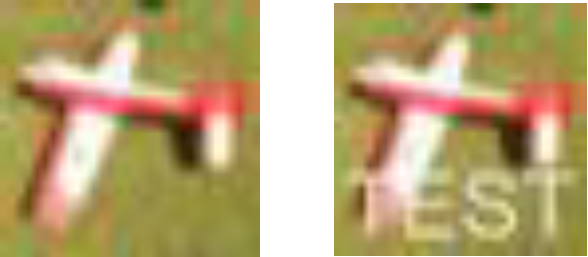}
\caption{}\label{fig:watermark2}
\end{subfigure}
\begin{subfigure}{0.116\textwidth}
\centering
\includegraphics[width=\columnwidth]{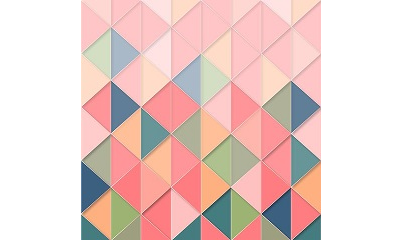}
\caption{}\label{fig:watermark3}
\end{subfigure}
\end{center}
% \vspace{-4mm}
\caption{\small Examples of watermarks generated by (a)-(b) pattern-based and (c) instance-based techniques.  }\label{fig:watermarks}
\end{figure}

%
% Specifically, instance-based techniques inject individual training samples labeled with a specific class as watermarks and the verification can be done by inputting the same samples into the trained model. 
% pattern-based techniques inject a set of samples that are blended with the same pattern and labeled with one certain class into the training set; the data contributor can later verify the data source of the trained model by checking the output of the model for an input with the pattern.
In this application, we demonstrate that it is always possible for the model trainer to remove the watermarks based on data importance. The idea is that the watermarks should have low data importance by nature, since they contribute little to predict the normal validation data. Note that this experiment constitutes a new type of attack, which might be of independent interest itself. 

% \begin{wrapfigure}{r}{0.3\textwidth}
% %\vspace{-10mm}
% \begin{center}
% \includegraphics[width=\linewidth]{figs/Watermark (a).pdf}
% \end{center}
% %\vspace{-5mm}
% \end{wrapfigure}

We consider the setting of one logistic regression model trained on $10000$ images from MNIST for instance-based watermark removal. The watermark ratio is $10\%$.
For this experiment, we found that both watermarks and benign instances tend to have low scores on some validation instances; therefore, they are not quite separable in terms of the score averaged over the whole validation set. We instead propose to compute the max score across the validation set for each training point, which we name as \textit{max-$K$NN-Shapley}, and remove the instances with lowest \textit{max-$K$NN-Shapley} values. The intuition is that out-of-distribution samples are inessential to the prediction of normal validation instances and thus the maximum of their Shapley values w.r.t. different validation instances should be low. 
In plotting Fig.~\ref{fig:exp-2}, we examine the label of the training instances that have the lowest scores and plot the change of the fraction of the detected watermarks (in percentage) with the fraction of the checked training data (in percentage). The figure reveals that our \textit{max-$K$NN-Shapley} is more effective in detecting instance-based watermarks than all other baselines. 

%The detailed elaborations and results on the instance-based watermark removal task can be found in Appendix~\ref{appendix:sec_removal}.

% \vspace{-1.5em}
\paragraph{Data Summarization.}
\phantomsection
\label{sec:ds}
Data summarization aims to select a small representative subset from a massive dataset, which can retain a comparable utility to that of the whole dataset. This is a natural application of data importance, since we can directly reduce the dataset size by eliminating data of low importance. 

% \begin{wrapfigure}{r}{0.3\textwidth}
% %\vspace{-3em}
% \begin{center}
% \includegraphics[width=0.3\textwidth]{figs/DS_UCI.pdf}
% \end{center}
%   %\vspace{-5mm}
% \end{wrapfigure}

We use a single hidden layer neural network trained on UCI Adult Census dataset. 
% The training set is composed of $1000$ randomly picked individuals, where half of them have income exceeding $\$50,000$ per year. We use another balanced dataset of size $500$ to calculate the scores of training data and a held-out dataset containing $1000$ data points to evaluate the model performance. 
% In the experiment, we gradually remove low-value points from the ground training set and evaluate the accuracy of the model trained with remaining data. 
In Fig.~\ref{fig:exp-3}, we plot the change of prediction accuracy (in percentage) with the change of the fraction of data removed (in percentage). The figure reveals that the instances selected by the Shapley value--based data importance measures are more representative than the LOO-based measures. Though TMC-Shapley and G-Shapley can achieve slightly better performance than $K$NN-Shapley, our method still retains a high performance even after reducing $50\%$ of the whole training set, which is notable.

Apart from the quantitative results above, we provide the qualitative visualization of images drawn from Tiny ImageNet in Fig.~\ref{fig:ds-viz}, where we show the images of the highest Shapley value (\ie, representative images), as well as the top 50 classes that their summarization belongs to.
% It worth to point out that all classes have an equal number of images in the full dataset, but after our summarization step, certain classes retain 
It is intriguing that a similar set of images (\eg, dugong, espresso, monarch, goldfish) stand out as the most representative samples even when they are pre-processed using different feature extractors. Why these classes are more representative is an interesting open question that deserves further investigation.
Another observation is the high diversity of the top 20 images displayed, which further corroborates the capability of our Shapley enriched method in producing a high-quality miniature for the original massive dataset.

%In Appendix~\ref{appendix:sec_ds}, we provide results on Tiny Imagenet.

% \vspace{-1.5em}
\paragraph{Active Data Acquisition.}
\phantomsection
\label{sec:da}

Annotated data is often hard and expensive to obtain, particularly for specialized domains where only experts can provide reliable labels. Active data acquisition aims to facilitate the data collection process by automatically deciding which instances an annotator should label to train a model. To simulate this scenario, we start with a small training set, and then train a random forest to predict the score for new data based on their features. We repeat the process and iteratively add new data with highest data importance to the training set.

% \begin{wrapfigure}{r}{0.3\textwidth}
% %\vspace{-8mm}
% \begin{center}
% \includegraphics[width=0.3\textwidth]{figs/New_DA_Mnist.pdf}
% \end{center}

%   %\vspace{-5mm}
% \end{wrapfigure}

%To better reveal the capability of our model, we target at datasets with distinct data qualities. Tiny ImageNet is one good example due to its realistic variation of data quality. We attach the results on this dataset in Appendix~\ref{appendix:sec_da}.

%Another possible scenario is when data quality is undermined by the injected noise. 
Here, we choose MNIST as our dataset and inject noise to part of it.
We start with a small training set with $100$ images and add Gaussian white noise into half of them. We use another $100$ images to calculate the scores of training data and a held-out validation dataset of size $1000$ to evaluate the performance.
In Fig.~\ref{fig:exp-4} we plot the change of prediction accuracy with the number of added training points. Evidently, new data selected based on $K$NN-Shapley value improves model accuracy faster than all other methods.

\begin{table}[]
\scriptsize
    \centering
    \caption{\small Domain adaptation between MNIST and USPS.}\label{tab:da}
\begin{tabular}{crr} 
\toprule \textbf{Method} & \textbf{MNIST} $\rightarrow{}$ \textbf{USPS} & \textbf{USPS} $\rightarrow{}$ \textbf{MNIST} \\ 
& 
\includegraphics[width=4mm]{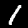}
\includegraphics[width=4mm]{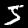} $\rightarrow{}$
\includegraphics[width=4mm]{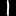}
\includegraphics[width=4mm]{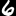}
& 
\includegraphics[width=4mm]{./domainAdaptation_data/usps2.png}
\includegraphics[width=4mm]{./domainAdaptation_data/usps1.png} $\rightarrow{}$
\includegraphics[width=4mm]{./domainAdaptation_data/mnist2.png}
\includegraphics[width=4mm]{./domainAdaptation_data/mnist1.png}
\\\midrule
% \hline
$K$NN-Shapley &  \textbf{31.70\% $\rightarrow{}$ 47.00\%}  &  \textbf{23.35\% $\rightarrow{}$ 29.80\%}   \\ 
$K$NN-LOO &  31.70\% $\rightarrow{}$ 37.40\%  &  23.35\% $\rightarrow{}$ 24.50\% \\ 
TMC-Shapley &  31.70\% $\rightarrow{}$ 44.90\%  &  23.35\% $\rightarrow{}$ 29.55\%  \\ 
LOO & 31.70\% $\rightarrow{}$ 29.40\% & 23.35\% $\rightarrow{}$ 23.53\% \\
\bottomrule
\end{tabular}
\end{table}

% \vspace{-1.5em}
\paragraph{Domain Adaptation.}
\phantomsection
\label{sec:adaptation}
% Machine learning models are known to have limited capability of generalizing learned knowledge to new datasets or environments. In practice, there is a need to transfer the model from a source domain where sufficient training data is available to a target domain where few labeled data are available. 
Domain adaptation aims to leverage the dataset from one domain for the prediction tasks in another domain. 
We will show that the data importance measures will be useful for domain adaptation.
Specifically, we first compute the importance of data in the source domain with respect to a held-out set from the target domain. 
We then train the model using only positive-valued points in source domain and evaluate the model in target domain. 
% \begin{wrapfigure}{c}{\linewidth}
% \vspace{-4mm}

We perform experiments on MNIST and USPS following the setups in Ghorbani \textit{et al.}~\cite{ghorbani2019data} and present the transfer results between the two.
We first train a multinomial logistic regression classifier. We randomly sample $1000$ images from the source domain as the training set, calculate the scores for the training data based on $1000$ instances from the target domain, and evaluate the performance of the model on another $1000$ target domain instances. The results are summarized in Table~\ref{tab:da}. As it shows, $K$NN-Shapley performs the best. 

\vspace{-1.5em}
\paragraph{Summary of Results.}
Based on extensive empirical observations, we conclude that: (1) the $K$NN-Shapley-based method requires the minimal runtime compared with the rest approaches on large scale training data and models (some methods such as TMC-Shapley cannot even finish running within reasonable time); (2) for different ML applications (e.g. mislabeled data detection, watermark removal, data summarization active data acquisition, and domain adaptation), different variants of Shapley-based methods consistently outperform the leave-one-out--based methods; (3) the $K$NN-Shapley based methods including both \textit{$K$NN-} and \textit{max-$K$NN-Shapley}, always achieve the best or at least comparable performance compared to other Shapley approximation methods.

\subsection{Comparisons of Different Embeddings}
\label{subsec:embedding}

% \begin{wrapfigure}{r}{0.3\textwidth}
% \vspace{-10mm}
% \begin{center}
% \includegraphics[width=0.3\textwidth]{figs/Label_embedding.pdf}
% \end{center}
% \vspace{-5mm}
% \end{wrapfigure}

In Section~\ref{subsec:comp} we provide results corresponding to the embeddings extracted by  MobileNett~\cite{howard2017mobilenets} classifier pre-trained on ImageNet.
In this section, we leverage different embeddings extracted by 4 other pre-trained classifiers for evaluation: ResNet18~\cite{he2016deep}, VGG11~\cite{Simonyan15vgg}, Inception-V3~\cite{szegedy2015inception}, and EfficientNet B7~\cite{tan2019efficientnet}.
% , and provide the comparison of performance.

We provide part of results in Fig.~\ref{fig:embedding-main} for selected applications and datasets. 
In each figure, there are clearly two groups of curves: one group for $K$NN-Shapley and the other for $K$NN-LOO.
% Notably, there are two groups of curves in each figure: 
% the curves of $K$NN-Shapley are close to each other, so are the curves for $K$NN-LOO. 
We can see that the utility of $K$NN-LOO is roughly the same as random, while the $K$NN-Shapley presents high utility for different applications.
In conclusion, the difference induced by using different \textit{embeddings} is marginal compared to using different \textit{measures}. Furthermore, our $K$NN-Shapley based importance measure is insensitive to the selection of embeddings and can achieve superb performance without the need of carefully selecting feature extractors. 
A more comprehensive set of results are left to Appendix~\ref{appendix:sec_embedding}.
% We provide a comprehensive set of results in Fig.~\ref{fig:embedding}, where similar conclusions can be drawn.

\begin{figure}
\begin{center}
\begin{subfigure}{0.49\linewidth}
\centering
\includegraphics[width=\columnwidth]{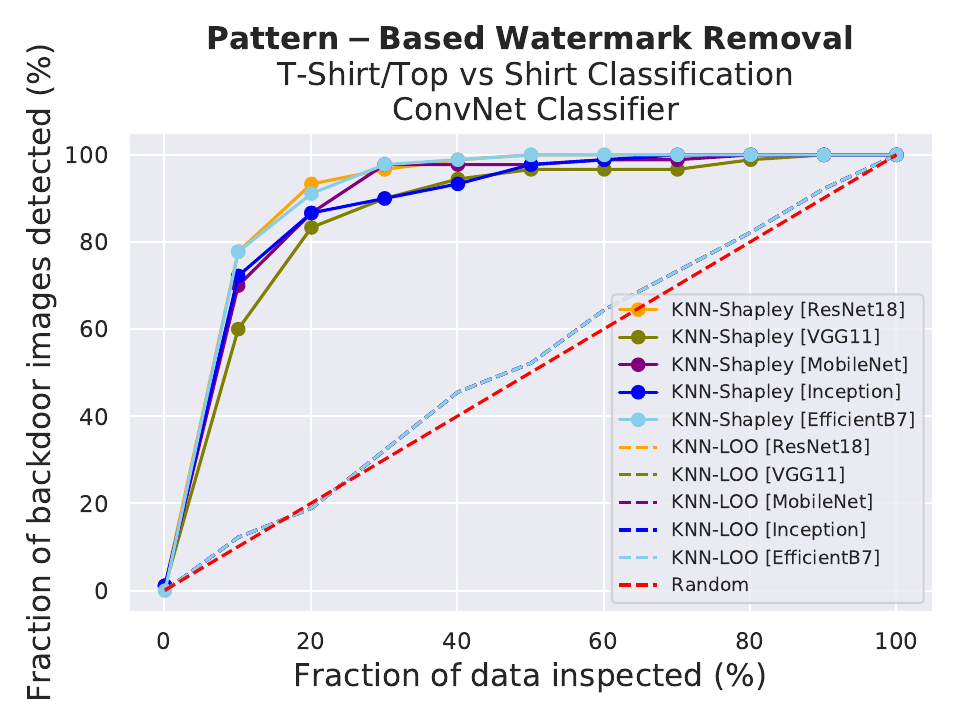}
\end{subfigure}
\begin{subfigure}{0.49\linewidth}
\centering
\includegraphics[width=\columnwidth]{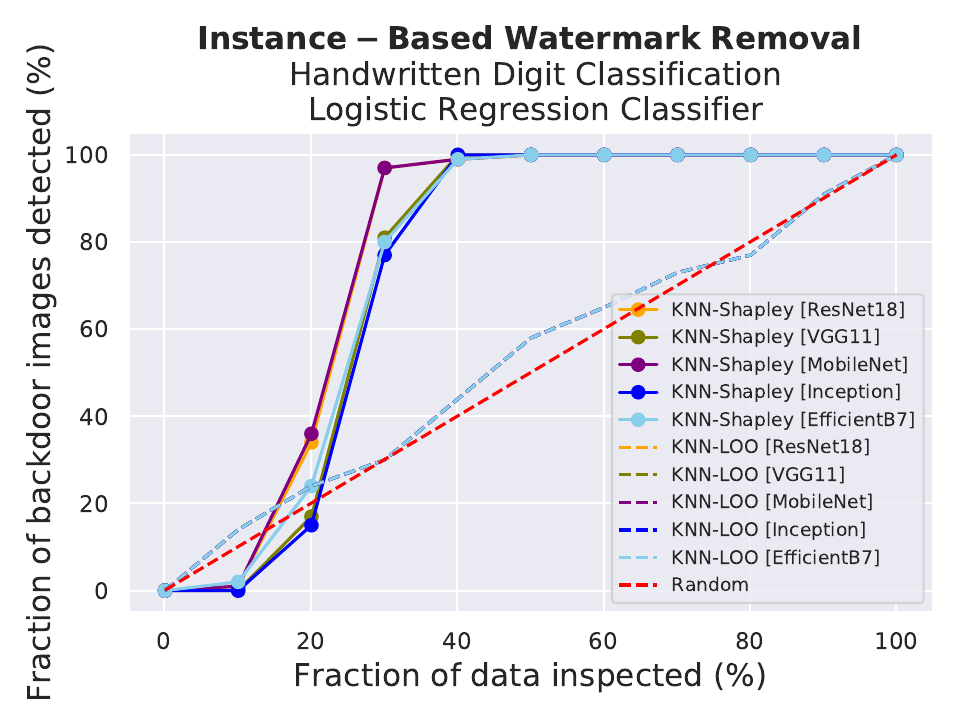}
\end{subfigure}
\begin{subfigure}{0.49\linewidth}
\centering
\includegraphics[width=\columnwidth]{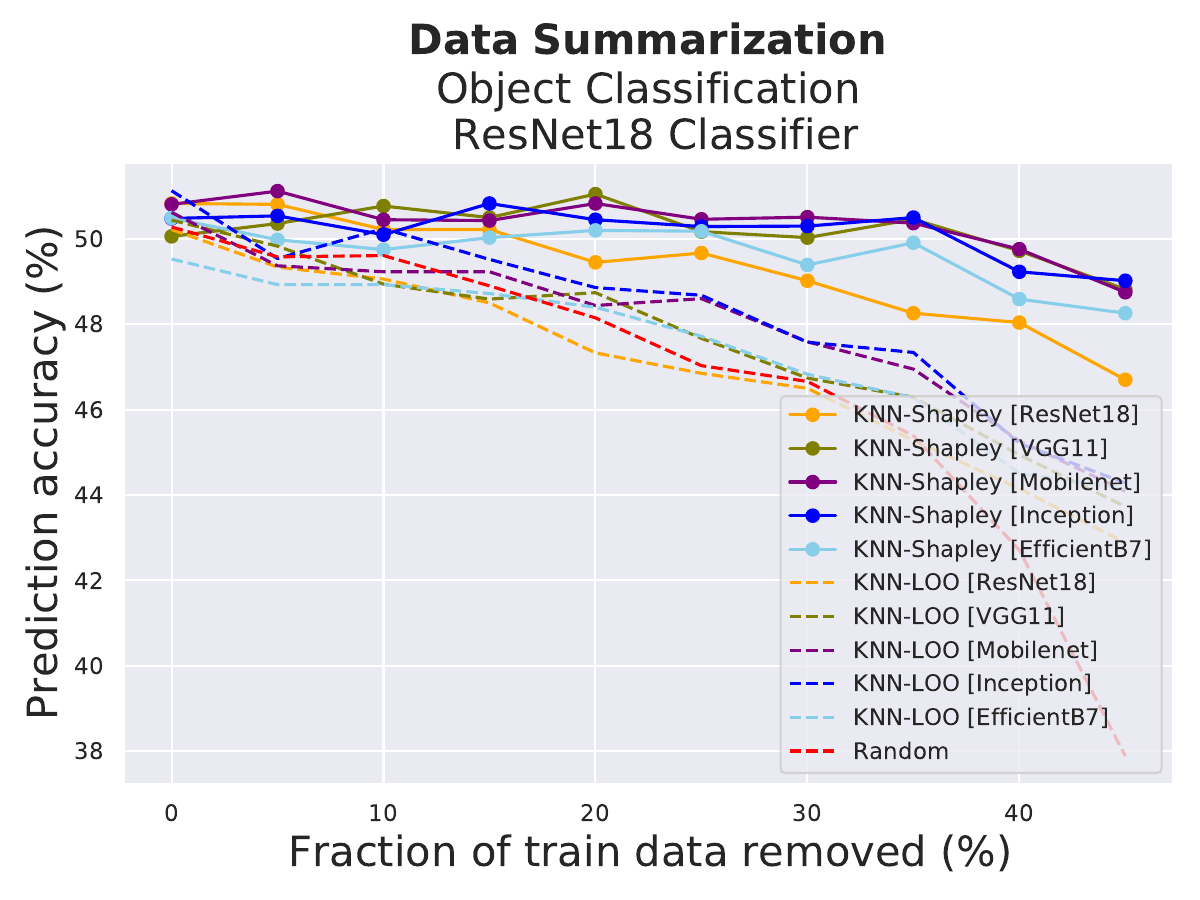}
\end{subfigure}
\begin{subfigure}{0.49\linewidth}
\centering
\includegraphics[width=\columnwidth]{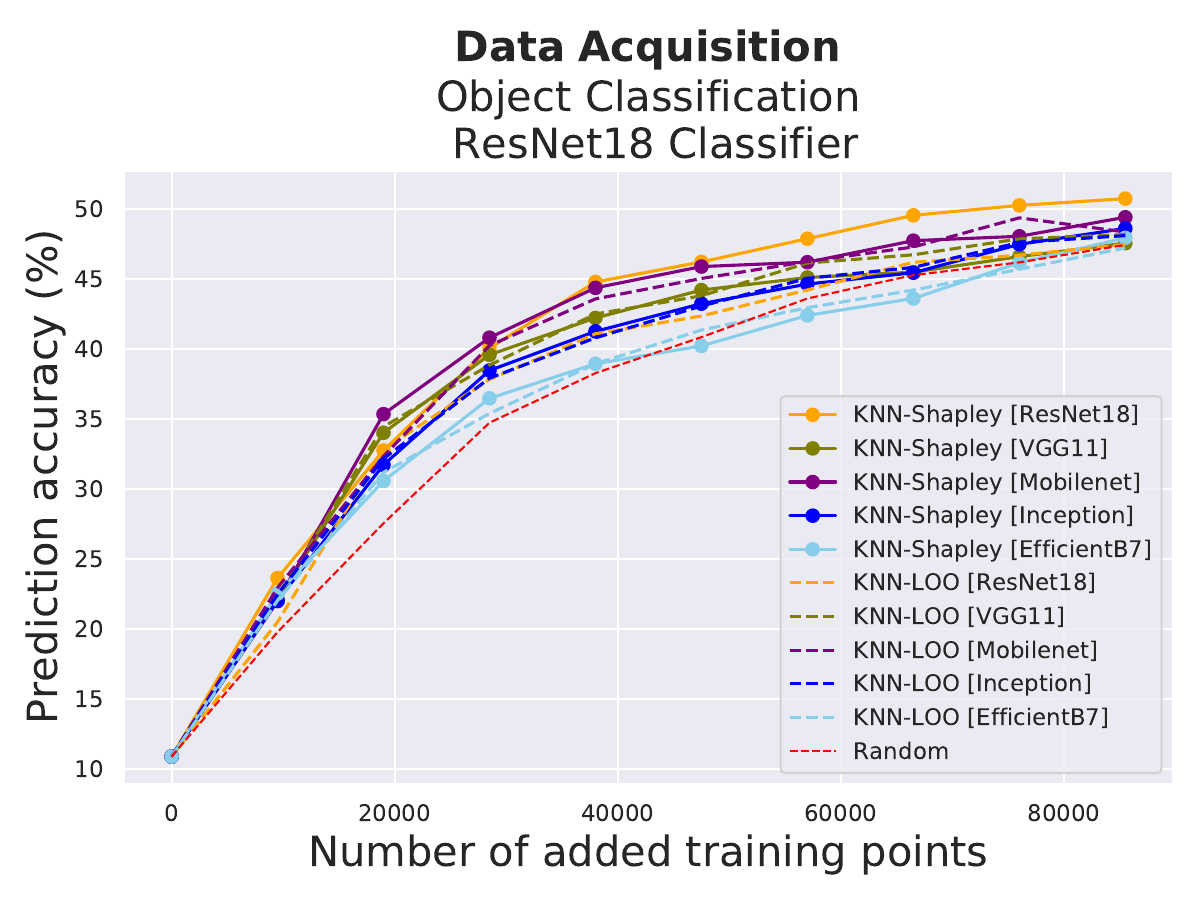}
\end{subfigure}
\end{center}
\vspace{-2mm}
\caption{\small Comparisons of different embeddings on different applications and datasets.}\label{fig:embedding-main}
\end{figure}

%\vspace{-0.5em}
\section{Conclusion}
This paper provides the first theoretical and large-scale empirical studies towards answering the fundamental questions about what method should be used for evaluating data importance and how to efficiently do so. 
Particularly, we prove that the Shapley-based method provides higher utility than a leave-one-out--based approach, in terms of evaluating the predictive power of the data importance as well as the data discrimination ability. 
Extensive experiments are conducted on five applications, showing that the Shapley-based methods outperform the leave-one-out--based ones in terms of both runtime and experimental performance. Specifically, the $K$NN-Shapley approach provides the most efficient solution and usually achieves the best or comparable performance among all. 
In addition, we are the first to leverage data importance approaches to perform watermark removal, which is a challenging task currently, and achieve promising results. This particular application would shed light on future research on watermark analysis and other related tasks.

\noindent\textbf{Acknowledgement} 
\begin{small}
This work was performed under the auspices of the NSF grant \#1910100 and U.S. Department of Energy by the Lawrence Livermore National Laboratory under Contract No. DE-AC52-07NA27344, Lawrence Livermore National Security, LLC. The views and opinions of the authors do not necessarily reflect those of the U.S. government or Lawrence Livermore National Security, LLC neither of whom nor any of their employees make any endorsements, express or implied warranties or representations or assume any legal liability or responsibility for the accuracy, completeness, or usefulness of the information contained herein. LLNL-CONF-820505.
\end{small}

{\small
\bibliographystyle{ieee_fullname}
\bibliography{neurips2020_conference}
}

\input{appendix.tex}

\end{document}

%% file: math_commands.tex
\usepackage{amsmath,amsfonts,bm}

\def\eqref#1{equation~\ref{#1}}
\def\1{\bm{1}}

\DeclareMathAlphabet{\mathsfit}{\encodingdefault}{\sfdefault}{m}{sl}
\SetMathAlphabet{\mathsfit}{bold}{\encodingdefault}{\sfdefault}{bx}{n}

\newcommand{\E}{\mathbb{E}}

\newcommand{\A}{\mathcal{A}}
\newcommand{\D}{\mathcal{D}}
\newcommand{\OO}{\mathcal{O}}

%% file: appendix.tex
% \documentclass{article}
% \usepackage[nonatbib]{neurips_2020}

% % Recommended, but optional, packages for figures and better typesetting:
% % \usepackage{microtype}
% % \usepackage{graphicx}
% % \usepackage{subfigure}
% % \usepackage{booktabs} % for professional tables

% % hyperref makes hyperlinks in the resulting PDF.
% % If your build breaks (sometimes temporarily if a hyperlink spans a page)
% % please comment out the following usepackage line and replace
% % \usepackage{icml2020} with \usepackage[nohyperref]{icml2020} above.
% % \usepackage{hyperref}

% % Attempt to make hyperref and algorithmic work together better:
% % \newcommand{\theHalgorithm}{\arabic{algorithm}}

% % Use the following line for the initial blind version submitted for review:
% % \usepackage{icml2020}

% \usepackage{enumitem}
\setlist[enumerate]{itemsep=0mm}

\onecolumn

\appendix

\section{Proof of Theorem 2}

%\goldbach*

\begin{proof}
The proof relies on dissecting the term $\E[U(T\cup \{z_i\}) - U(T\cup \{z_j\})]$ and $\nu(z_i) - \nu(z_j)$ ($\nu=\nu_\text{shap-knn}, \nu_\text{LOO-knn}$) in the definition of order-preserving property.

Consider any two points $z_i,z_{i+l}\in D$. We start by analyzing $\E[U(T\cup \{z_i\}) - U(T\cup \{z_{i+l}\})]$. Let the $k$th nearest neighbor of $x_\text{val}$ in $T$ be denoted by $T_{(k)} = (x_{(k)},y_{(k)})$. Moreover, we will use $T_{(k)}\leq_d z_i$ to indicate that $x_{(k)}$ is closer to the validation point than $x_i$, i.e.,  $d(x_{(k)},x_\text{val})\leq d(x_i,x_\text{val})$. We first analyze the expectation of the above utility difference by considering the following cases:

(1) $T_{(K)}\leq_d z_i$. In this case, adding $z_i$ or $z_{i+l}$ into $T$ will not change the $K$-nearest neighbors to $z_\text{val}$ and therefore $U(T\cup \{z_i\}) = U(T\cup\{z_{i+l}\}) = U(T)$. Hence, $U(T\cup \{z_i\}) - U(T\cup\{z_{i+l}\}) = 0$.

(2) $z_i<_d T_{(K)}\leq_d z_{i+l}$. In this case, including the point i into T can expel the Kth nearest neighbor from the original set of K nearest neighbors while including the point $i+1$ will not change the $K$ nearest neighbors. In other words, $U(T\cup \{z_i\}) - U(T) = \frac{\mathbbm{1}[y_i = y_\text{val}] - \mathbbm{1}[y_{(K)} = y_\text{val} ]}{K}$ and $U(T\cup\{z_{i+l}\}) - U(T) = 0$. Hence, $U(T\cup \{z_i\}) - U(T\cup\{z_{i+l}\}) = \frac{\mathbbm{1}[y_i = y_\text{val}] - \mathbbm{1}[y_{(K)} = y_\text{val} ]}{K}$.

(3) $T_{(K)} >_d z_{i+l} $. In this case, including the point $i$ or $i+1$ will both change the original $K$ nearest neighbors in $T$ by excluding the $K$th nearest neighbor. Thus, $U(T\cup \{z_i\}) - U(T) = \frac{\mathbbm{1}[y_i = y_\text{val}] - \mathbbm{1}[y_{(K)} = y_\text{val} ]}{K}$ and $U(T\cup\{z_{i+l}\}) - U(T) = \frac{\mathbbm{1}[y_{i+l} = y_\text{val}] - \mathbbm{1}[y_{(K)} = y_\text{val} ]}{K}$. It follows that $U(T\cup \{z_i\}) - U(T\cup\{z_{i+l}\}) = \frac{\mathbbm{1}[y_i = y_\text{val}] - \mathbbm{1}[y_{i+l} = y_\text{val} ]}{K}$.

Combining the three cases discussed above, we have
\begin{align}
    &\E[U(T\cup \{z_i\}) - U(T\cup\{z_{i+l}\})]\\
    &=  P(T_{(K)}\leq_d z_i)\times 0 + P(z_i<_dT_{(K)}\leq_d z_{i+l})  \frac{\mathbbm{1}[y_i = y_\text{val}] - \mathbbm{1}[y_{(K)} = y_\text{val} ]}{K} \nonumber\\
    &\quad+ P(T_{(K)}>_d z_{i+l}) \frac{\mathbbm{1}[y_i = y_\text{val}] - \mathbbm{1}[y_{i+l} = y_\text{val}]}{K}\\
    \label{eqn:sum}
    & = P(z_i<_dT_{(K)}\leq_d z_{i+l})  \frac{\mathbbm{1}[y_i = y_\text{val}] - \mathbbm{1}[y_{(K)} = y_\text{val} ]}{K} \nonumber\\
    &\quad + P(T_{(K)}>_d z_{i+l}) \frac{\mathbbm{1}[y_i = y_\text{val}] - \mathbbm{1}[y_{i+l} = y_\text{val}]}{K}
\end{align}

Note that removing the first term in (\ref{eqn:sum}) cannot change the sign of the sum in (\ref{eqn:sum}). Hence, when analyzing the sign of (\ref{eqn:sum}), we only need to focus on the second term:
\begin{align}
\label{eqn:sign_exp}
   &P(T_{(K)}>_d z_{i+1}) \frac{\mathbbm{1}[y_i = y_\text{val}] - \mathbbm{1}[y_{i+l} = y_\text{val}]}{K}
\end{align}
Since $P(T_{(K)}>_d z_{i+1}) = \sum_{k=N-K+1}^N P(Z>_d z_{i+1})^k$, the sign of (\ref{eqn:sign_exp}) will be determined by the sign of $ \mathbbm{1}[y_i = y_\text{val}] - \mathbbm{1}[y_{i+l} = y_\text{val}]$. Hence, we get 
\begin{align}
\label{eqn:main}
    \big(\E[U(T\cup \{z_i\}) - U(T\cup\{z_{i+1}\})]\big)\times \big(\mathbbm{1}[y_i = y_\text{val}] - \mathbbm{1}[y_{i+l} = y_\text{val}]\big) > 0 
\end{align}

Now, we switch to the analysis of the value difference. By Theorem 1, it holds for \emph{the $K$NN-Shapley value} that
\begin{align}
\label{eqn:shapley_diff}
    &\nu_\text{shap-knn}(z_i) - \nu_\text{shap-knn}(z_{i+l}) \\
    &= \sum_{j=i}^{i+l-1}\frac{\min\{K,j\}}{jK} \big(\mathbbm{1}[y_{j} = y_\text{val}] - \mathbbm{1}[y_{j+1} = y_\text{val}]\big)\\
    \label{eqn:knn_shap_dissect}
    & =\frac{\min\{K,i\}}{iK} \mathbbm{1}[y_{i} = y_\text{val}] + \sum_{j=i}^{i+l-2} \bigg(\frac{\min\{K,j+1\}}{(j+1)K} - \frac{\min\{K,j\}}{jK}\bigg) \mathbbm{1}[y_{j+1} = y_\text{val}]\nonumber\\
    &\quad - \frac{\min\{K,i+l-1\}}{(i+l-1)K}\mathbbm{1}[y_{i+l} = y_\text{val}]
\end{align}
Note that $\frac{\min\{K,j+1\}}{(j+1)K} - \frac{\min\{K,j\}}{jK}<0$ for all $j=i,\ldots,i+l-2$. Thus, if $\mathbbm{1}[y_{i} = y_\text{val}] = 1$ and $\mathbbm{1}[y_{i+l} = y_\text{val}] = 0$, the minimum of (\ref{eqn:knn_shap_dissect}) is achieved when $ \mathbbm{1}[y_{j+1} = y_\text{val}] = 1$ for all $j = i,\ldots,i+l-2$ and the minimum value is $\frac{\min\{K,i+l-1\}}{(i+l-1)K}$, which is greater than zero. On the other hand, if $\mathbbm{1}[y_{i} = y_\text{val}] = 0$ and $\mathbbm{1}[y_{i+l} = y_\text{val}] = 1$, then the maximum of (\ref{eqn:knn_shap_dissect}) is achieved when $ \mathbbm{1}[y_{j+1} = y_\text{val}] = 0$ for all $j = i,\ldots,i+l-2$ and the maximum value is $-\frac{\min\{K,i+l-1\}}{(i+l-1)K}$, which is less than zero.

Summarizing the above analysis, we get that $\nu_\text{shap-knn}(z_i) - \nu_\text{shap-knn}(z_{i+l})$ has the same sign as $\mathbbm{1}[y_i = y_\text{val}] - \mathbbm{1}[y_{i+l} = y_\text{val}]$. By (\ref{eqn:main}), it follows that $\nu_\text{shap-knn}(z_i) - \nu_\text{shap-knn}(z_{i+l})$ also shares the same sign as $\E[U(T\cup \{z_i\}) - U(T\cup\{z_{i+1}\})]$.

To analyze the sign of \emph{the $K$NN-LOO value} difference, we first write out the expression for \emph{the $K$NN-LOO value} difference:

\begin{align}
\label{eqn:loo_diff}
    \nu_\text{loo-knn}(z_i) - \nu_\text{loo-knn}(z_{i+l}) =\left\{
    \begin{array}{cc}
        \frac{1}{K} (\mathbbm{1}[y_{i} = y_\text{val}] - \mathbbm{1}[y_{i+l} = y_\text{val}] )& \text{if $i+l\leq K$} \\
       \frac{1}{K} (\mathbbm{1}[y_{i} = y_\text{val}] - \mathbbm{1}[y_{K+1} = y_\text{val}]) & \text{if $i\leq K < i+l$} \\
        0 & \text{if $i> K$}
    \end{array}
    \right.
\end{align}

Therefore, $\nu_\text{loo-knn}(z_i) - \nu_\text{loo-knn}(z_{i+l})$ has the same sign as $\mathbbm{1}[y_{i} = y_\text{val}] - \mathbbm{1}[y_{i+l} = y_\text{val}] $ and $\E[U(T\cup \{z_i\}) - U(T\cup\{z_{i+1}\})]$ only when $i+l\leq K$. 

% Based on (\ref{eqn:main}), (\ref{eqn:shapley_diff}), and (\ref{eqn:loo_diff}), we can obtain the result in the theorem.

\end{proof}

\section{Proof of Theorem 3}

We will need the following lemmas on group differential privacy for the proof of Theorem 3.

\begin{Lm}
\label{lm:group}
If $\A$ is $(\epsilon,\delta)$-differentially private with respect to one change in the database, then $\A$ is $(c\epsilon,ce^{c\epsilon}\delta)$-differentially private with respect to $c$ changes in the database.
\end{Lm}

\begin{Lm}[\cite{jia2019towards}]
\label{lm:shapley_diff}
For any $z_i,z_j\in D$, the difference in Shapley values between $z_i$ and $z_j$ is 
\begin{align}
\label{eqn:shapley_diff}
    \nu_\text{shap}(z_i) - \nu_\text{shap}(z_j) = \frac{1}{N-1} \!\! \sum_{T\subseteq D\setminus\{z_i,z_j\}} \!\!\!\!\frac{U(T\cup\{z_i\}) - U(T\cup \{z_j\})}{\binom{N-2}{|T|}}
\end{align}
\end{Lm}

% \differentialprivacy*

\begin{proof}
Let $S'$ be the set with one element in $S$ replaced by a different value. Let the probability density/mass defined by $\mathcal{A}(S')$ and $\mathcal{A}(S)$ be $p(h)$ and $p'(h)$, respectively. Using Lemma~\ref{lm:group}, for any $z_\text{val}$ we have
\begin{align}
    \E_{h\sim \A(S)} l(h,z_\text{val}) &= \int_{0}^1 P_{h\sim \A(S)} [l(h,z_\text{val})> t ]dt\\
    &\leq \int_0^1 (e^{c\epsilon} P_{h\sim \A(S')} [l(h,z_\text{val})> t ] + ce^{c\epsilon}\delta) dt\\
    & = e^{c\epsilon} \E_{h\sim \A(S')}[l(h,z_\text{val})] + ce^{c\epsilon}\delta
\end{align}
It follows that
\begin{align}
    \E_{h\sim \A(S)} l(h,z_\text{val}) - \E_{h\sim \A(S')}[l(h,z_\text{val})]& \leq (e^{c\epsilon} - 1)\E_{h\sim \A(S')}[l(h,z_\text{val})] + ce^{c\epsilon}\delta\\
    &\leq e^{c\epsilon} - 1 + ce^{c\epsilon}\delta
\end{align}

By symmetry, it also holds that
\begin{align}
    \E_{h\sim \A(S')} l(h,z_\text{val}) - \E_{h\sim \A(S)}[l(h,z_\text{val})]& \leq (e^{c\epsilon} - 1)\E_{h\sim \A(S)}[l(h,z_\text{val})] + ce^{c\epsilon}\delta\\
    &\leq e^{c\epsilon} - 1 + ce^{c\epsilon}\delta
\end{align}
Thus, we have the following bound:
\begin{align}
\label{eqn:group_change}
   | \E_{h\sim \A(S)} l(h,z_\text{val}) - \E_{h\sim \A(S')}[l(h,z_\text{val})]| \leq  e^{c\epsilon} - 1 + ce^{c\epsilon}\delta
\end{align}

Denoting $\epsilon' = e^{c\epsilon} - 1 + ce^{c\epsilon}\delta$. For the performance measure that evaluate the loss averaged across multiple validation points $U(S) = -\frac{1}{M}\sum_{i=1}^M \E_{h\sim \A(S)}l(h,z_{\text{val},i})$, we have
\begin{align}
    |U(S)-U(S')|\leq \epsilon'
\end{align}
Making the dependence on the training set size explicit, we can re-write the above equation as
\begin{align}
    \max_{z_i,z_j\in D, T\subseteq D\setminus\{z_i,z_j\}} |U(T\cup z_i) - U(T\cup z_j)| \leq \epsilon'(|T|+1)
\end{align}
By Lemma~\ref{lm:shapley_diff}, we have for all $z_i,z_j\in D$,
\begin{align}
    \nu_\text{shap}(z_i) - \nu_\text{shap}(z_j) &\leq  \frac{1}{N-1} \sum_{k=0}^{N-2} \sum_{T\subseteq D\setminus\{z_i,z_j\}, |T| =k} \frac{\epsilon'(k+1)}{\binom{N-2}{k}}\\
    & = \frac{1}{N-1} \sum_{k=0}^{N-2} \epsilon'(k+1)\\
    & = \frac{1}{N-1} \sum_{k=1}^{N-1} \epsilon'(k)
\end{align}

As for the LOO value, we have
\begin{align}
    \nu_{loo}(z_i ) - \nu_{loo}(z_j) &= U(D\setminus \{z_j\}) - U(D\setminus \{z_i\})\\
    &\leq \epsilon'(N-1)
\end{align}

\end{proof}

\section{Comparing the LOO and the Shapley Value for Stable Learning algorithms}
\label{appendix:stable}

An algorithm $G$ has uniform stability $\gamma$ with respect to the loss function $l$ if $\|l(G(S),\cdot) - l(G(S^{\setminus i}),\cdot)\|_\infty \leq \gamma$ for all $i\in \{1,\cdots,|S|\}$, where $S$ denotes the training set and $S^{\setminus i}$ denotes the one by removing $i$th element of $S$. 

\begin{Thm}
\label{thm:stable}
For a learning algorithm $\A(\cdot)$ with uniform stability $\beta = \frac{C_\text{stab}}{|S|}$, where $|S|$ is the size of the training set
and $C_\text{stab}$ is some constant. Let the performance measure be $U(S) = -\frac{1}{M}\sum_{i=1}^M l(A(S),z_{\text{val},i})$. Then,

\begin{align}
    \max_{z_i\in D} \nu_\text{loo}(z_i) - \nu_\text{loo}(z^*) \leq \frac{C_\text{stab}}{N-1}
\end{align}
and
\begin{align}
    \max_{z_i\in D} \nu_\text{shap}(z_i) - \nu_\text{shap}(z^*) \leq \frac{C_\text{stab} (1+\log(N-1))}{N-1}
\end{align}

\end{Thm}

\begin{proof}
By the definition of uniform stability, it holds that
\begin{align}
    \max_{z_,z_j\in D,T\subseteq D\setminus\{z_i,z_j\}} |U(T\cup\{z_i\}) - U(T\cup\{z_j\})|\leq \frac{C_\text{stab}}{|T|+1}
\end{align}
Using Lemma~\ref{lm:shapley_diff}, we have
we have for all $z_i,z_j\in D$,
\begin{align}
    &\nu_\text{shap}(z_i) - \nu_\text{shap}(z_j) \\
    &\leq  \frac{1}{N-1} \sum_{k=0}^{N-2} \sum_{T\subseteq D\setminus\{z_i,z_j\}, |T| =k} \frac{C_\text{stab}}{\binom{N-2}{k}(k+1)}\\
& = \frac{1}{N-1} \sum_{k=0}^{N-2}\frac{C_\text{stab}}{k+1}
\end{align}
Recall the bound on the harmonic sequences
\begin{align*}
    \sum_{k=1}^N \frac{1}{k} \leq 1 + \log (N)
\end{align*}
which gives us
\begin{align*}
    \nu_\text{shap}(z_i) - \nu_\text{shap}(z_j)\leq \frac{C_\text{stab}(1+\log(N-1))}{N-1}
\end{align*}

As for the LOO value, we have

\begin{align}
    \nu_{loo}(z_i ) - \nu_{loo}(z_j) &= U(D\setminus \{z_j\}) - U(D\setminus \{z_i\})\leq \frac{C_\text{stab}}{N-1}
\end{align}

\end{proof}

\section{Additional Experiments}

\subsection{Rank Correlation with Ground Truth Shapley Value}
{
We perform experiments to compare the ground truth Shapley value of raw data and the value estimates produced by different heuristics. The ground truth Shapley value is computed using the group testing algorithm in~\cite{jia2019towards}, which can approximate the Shapley value with provable error bounds. We use a fully-connected neural network with three hidden layers as the target model. Following the setting in~\cite{jia2019towards}, we construct a size-1000 training set using MNIST, which contains both benign and adversarial examples, as well as a size-100 validation set with pure adversarial examples. The adversarial examples are generated by the Fast Gradient Sign Method~\cite{goodfellow2014explaining}. This construction is meant to simulate data with different levels of usefulness. In the above setting, the adversarial examples in the training set should be more valuable than the benign data because they can improve the prediction on adversarial examples. Note that the $K$NN-Shapley computes the Shapley value of deep features extracted from the penultimate layer. 

The rank correlation of $K$NN-Shapley and G-Shapley with the ground truth Shapley value is 0.08 and 0.024 with p-value 0.0046 and 0.4466, respectively. It shows that both heuristics may not be able to preserve the exact rank of the ground truth Shapley value. Since TMC-Shapley cannot finish in a week for this model and data size, we omit it from comparison. We further apply some local smoothing to the scores and check whether these heuristics can produce large scores for data groups with large Shapley values. Specifically, we compute 1 to 100 percentiles of the Shapley values, find the group of data points within each percentile interval, and compute the average Shapley value as well as the average heuristic scores for each group. The rank correlation of the average $K$NN-Shapley and the average G-Shapley with the average ground truth Shapley value for these data groups are 0.22 and -0.002 with p-value 0.0293, 0.9843, respectively. We can see that although ignoring the data contribution for feature learning, $K$NN-Shapley can better preserve the rank of the Shapley value in a macroscopic level than G-Shapley. 
}

\section{Experiment Details and Results on More Datasets}
\label{appendix:sec_exp}
In this section, we present experiment details and results on more datasets corresponding to the applications introduced in the main body (see Section 4).

\subsection{Detecting Noisy Labels}
\label{appendix:sec_detect}

Following Ghorbani \textit{et al.}~\cite{ghorbani2019data}, we conducted another two experiments: a Naive Bayes model trained on a spam classification dataset and a logistic regression model trained on Inception-V3 features of a flower classification dataset. The noise flipping ratio is 20\% and 10\% respectively for these two datasets. The performance

\begin{figure}[htbp]
\begin{center}
\begin{subfigure}{0.255\textwidth}
\centering
\includegraphics[width=\columnwidth]{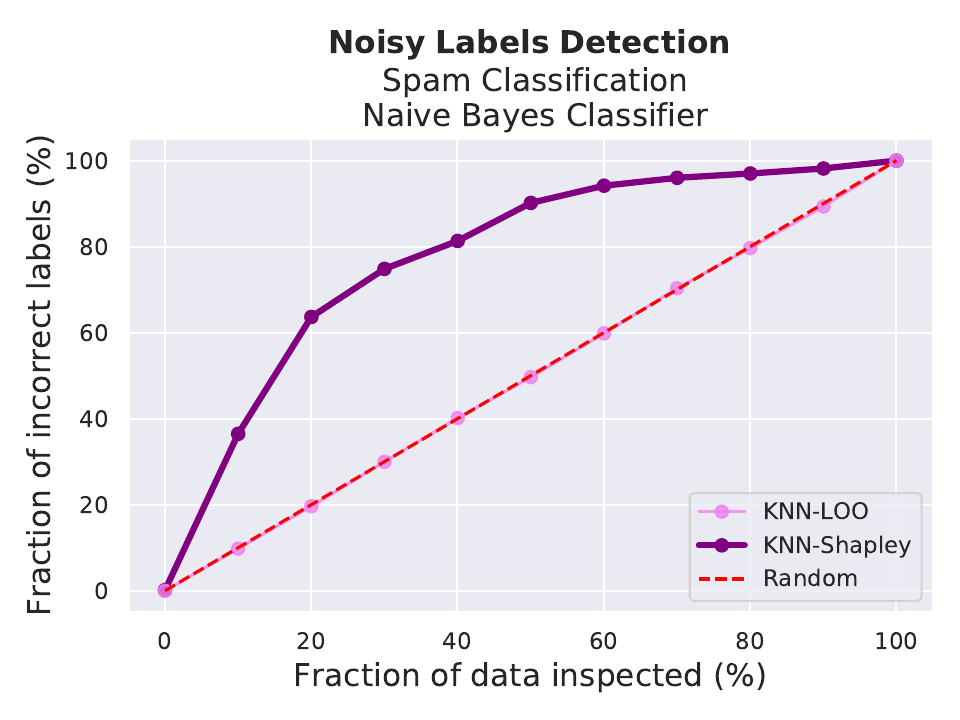}
\caption{}\label{fig:detection-a}
\end{subfigure}
\begin{subfigure}{0.255\textwidth}
\centering
\includegraphics[width=\columnwidth]{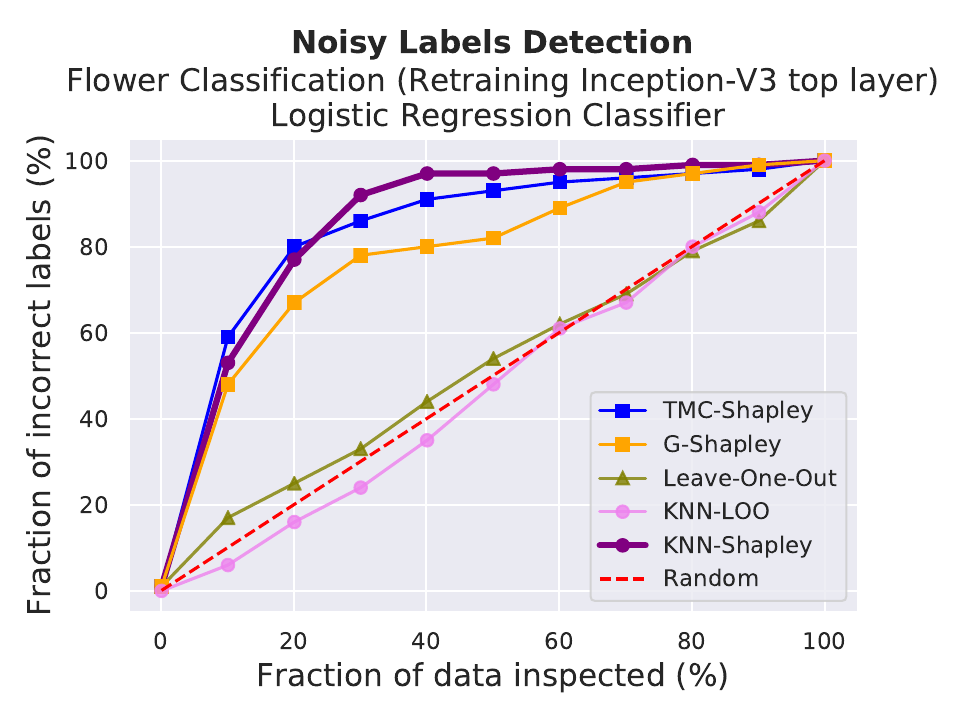}
\caption{}\label{fig:detection-b}
\end{subfigure}
\begin{subfigure}{0.267\textwidth}
\centering
\includegraphics[width=\columnwidth]{figs/watermark2.png}
\caption{}\label{fig:watermark2}
\end{subfigure}
\begin{subfigure}{0.2\textwidth}
\centering
\includegraphics[width=\columnwidth]{figs/watermark3.png}
\caption{}\label{fig:watermark3}
\end{subfigure}
\end{center}
\caption{\small (a-b) Results of noisy label detection on Spam Dataset and Flower Dataset; (c-d) Examples of watermarks generated by pattern-based techniques and instance-based techniques.  }\label{fig:overall-2}
\end{figure}

The performance of different data importance measures is illustrated in Fig.~\ref{fig:detection-a} and Fig.~\ref{fig:detection-b}. We examine the label of the training instances that have the lowest scores, and plot the change of the fraction of detected mislabeled data with the fraction of the checked training data. We can see that \emph{the $K$NN-Shapley value} outperforms all other methods. Also, the Shapley value--based measures, including TMC-Shapley, G-Shapley, and our KNN-Shapley, are more effective than the LOO-based measures.

\subsection{Watermark Removal}
\label{appendix:sec_removal}

We discuss two main types of techniques for injecting watermarks. The pattern-based techniques inject a set of samples that are blended with the same pattern and labeled with one certain class into the training set; the data contributor can later verify the data source of the trained model by checking the output of the model for an input with the pattern. The instance-based techniques, by contrast, inject individual training samples labeled with a specific class as watermarks and the verification can be done by inputting the same samples into the trained model. 

% \begin{wrapfigure}{r}{0.47\textwidth}
% \centering
% {
% \subfloat[{\footnotesize Pattern-based}]{\label{fig:poison2b}\includegraphics[width=0.267\textwidth]{figs/watermark2.png}}
% \subfloat[{\footnotesize Instance-based}]{\label{fig:poison2c}\includegraphics[width=0.2\textwidth]{figs/watermark3.png}}
% \caption{\footnotesize Watermark Examples}
% }
% \vspace{-5mm}
% \end{wrapfigure}

In pattern-based watermark removal, we adopted two types of patterns: one is to change the pixel values at the corner of an image~\cite{chen2018detecting}, another is to blend a specific word (like ``TEST'') into an image, as shown in Figure~\ref{fig:watermark2}. Specifically, after an image is blended with the ``TEST'' pattern, there is high chance that it is classified as the target label, e.g, an ``automobile'' on CIFAR-10.
The first pattern is used in the experiments on fashion MNIST and MNIST, which is composed of single channel images. The second pattern is applied to Pubfig-83 which contains multi-channel images.

In instance-based watermark removal, we used the same watermarks as~\cite{adi2018turning}, which contains a set of abstract images with specific assigned labels. The example of a trigger image is shown in Figure~\ref{fig:watermark3}. This type of watermarks are typically chosen from out-of-distribution data.

\begin{figure}[htbp]
\begin{center}
\begin{subfigure}{0.32\textwidth}
\centering
\includegraphics[width=\columnwidth]{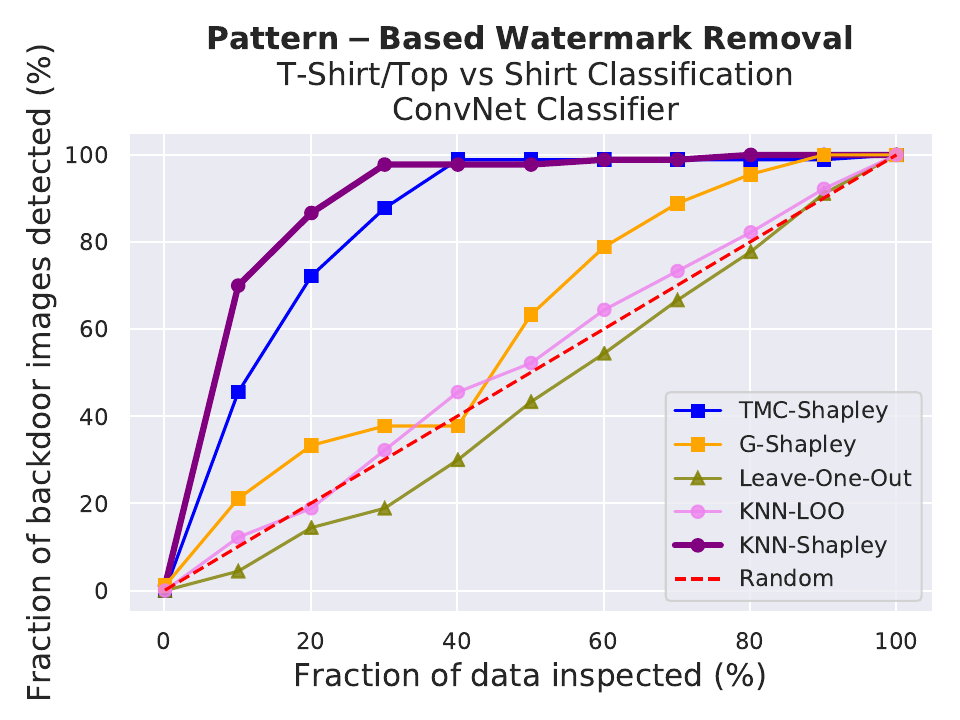}
\end{subfigure}
\begin{subfigure}{0.32\textwidth}
\centering
\includegraphics[width=\columnwidth]{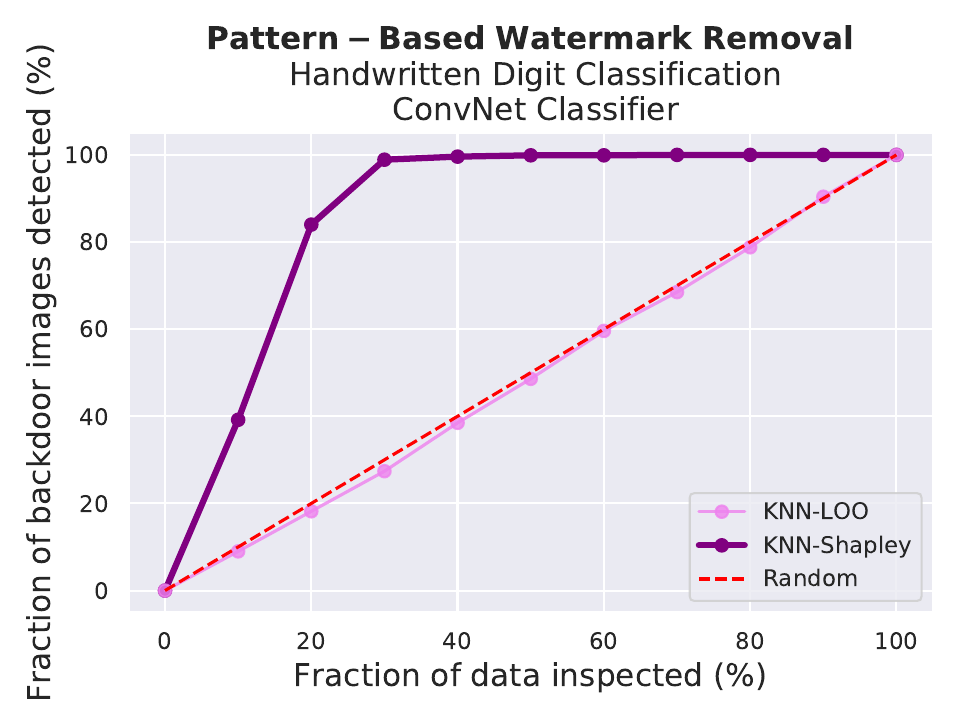}
\end{subfigure}
\begin{subfigure}{0.32\textwidth}
\centering
\includegraphics[width=\columnwidth]{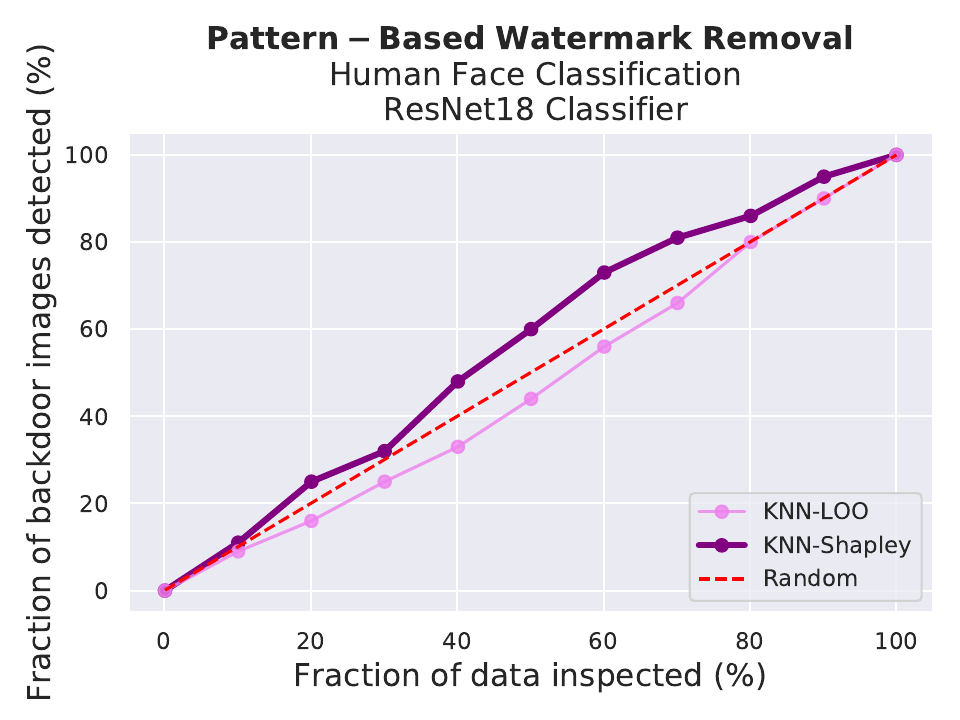}
\end{subfigure}
\end{center}
\caption{\small Results of pattern-based watermark removal tasks on (a) Fashion-MNIST Dataset~\cite{xiao2017fashion}, (b) MNIST Dataset~\cite{lecun2010mnist}, and (c) PubFig-83 Dataset~\cite{kumar2009pubfig}.}\label{fig:poison}
\end{figure}

For the pattern-based watermark removal experiment, we consider three settings: two convolutional networks trained on 1000 images from fashion MNIST and 10000 images from MNIST, respectively, and a ResNet18~\cite{he2016deep} model trained on 1000 images from the face recognition dataset Pubfig-83. The watermark ratio is 10\% for all three settings. Since for the last two settings, TMC-Shapley, G-Shapley, and Leave-one-out all fail to produce importance estimates in 3 hours either due to large data size or model size, we compare our algorithm only with the rest of the baselines. 
In plotting Fig.~\ref{fig:poison}, we examine the label of the training instances that have the lowest scores and plot the change of the faction of the detected watermarks (in percentage) with the fraction of the checked training data (in percentage).
Although TMC-Shapley can achieve similar performance to KNN-Shapley, its time complexity is actually much higher than KNN-Shapley. Compared with all other baselines, our $K$NN-Shapley outperforms achieves the best performance.

For the instance-based watermark removal experiment, we consider the following two settings: a convolution network trained on 3000 images from CIFAR-10~\cite{cifar10}, and ResNet18 trained on 3000 images from SVHN~\cite{Netzer2011ReadingDI}. The watermark ratio is 3\% in both settings. The results of our experiment are displayed in Fig.~\ref{fig:instance-b} and Fig.~\ref{fig:instance-c}. We plot the change of the fraction of the detected watermarks (in percentage) with the fraction of the checked training data (in percentage).

\begin{figure}[htbp]
\begin{center}
\begin{subfigure}{0.32\textwidth}
\centering
\includegraphics[width=\columnwidth]{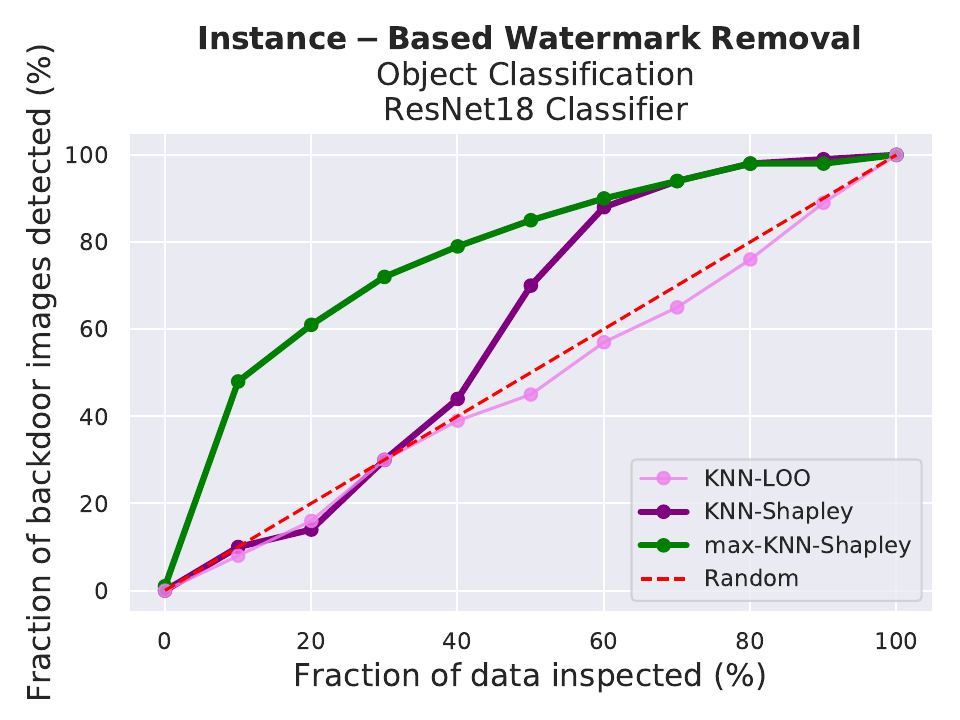}
\caption{}\label{fig:instance-b}
\end{subfigure}
\begin{subfigure}{0.32\textwidth}
\centering
\includegraphics[width=\columnwidth]{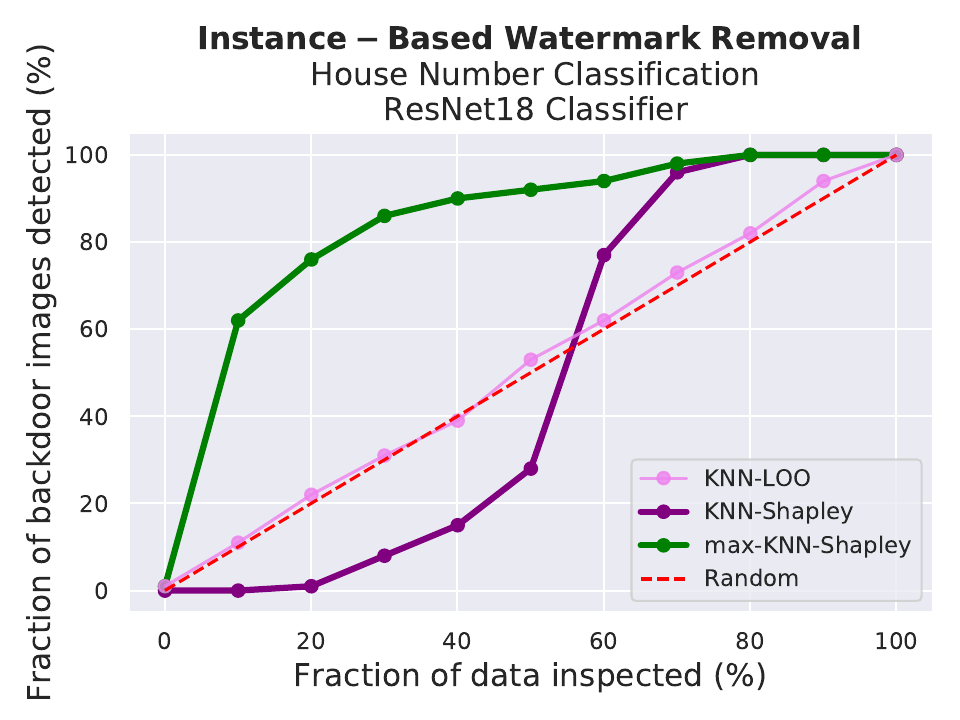}
\caption{}\label{fig:instance-c}
\end{subfigure}
\begin{subfigure}{0.32\textwidth}
\centering
\includegraphics[width=\columnwidth]{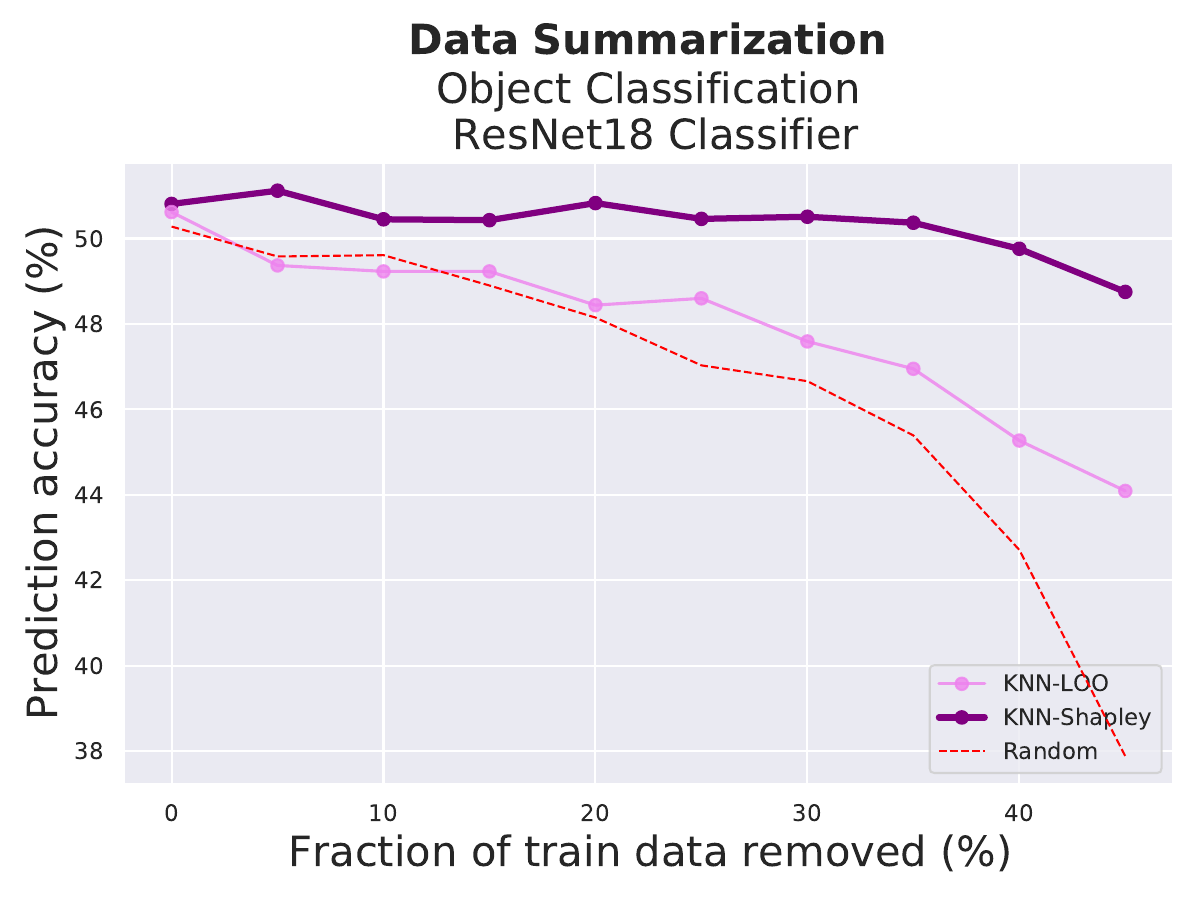}
\caption{}\label{fig:ds-tiny}
\end{subfigure}
\end{center}
\caption{\small Results of (a-b) instance-based watermark removal tasks on CIFAR-10~\cite{cifar10} and SVHN; (c) Data summarization on Tiny ImageNet~\cite{Le2015TinyIV}.}\label{fig:overall-3}
\end{figure}

As discussed in Section 4.3, we propose a novel measure max-$K$NN-Shapley to tackle the instance-based watermark removal task specifically. As shown in Fig.~\ref{fig:instance-b} and Fig.~\ref{fig:instance-c}, the max-$K$NN-Shapley is a more effective measure to detect instance-based watermarks than all other baselines. 

\begin{table}[H]
\centering
\footnotesize
\centering
\caption{\small 
% (a) 
\textbf{Instance-based watermark removal.} Prediction accuracy on different types of data.
% across different datasets.
% (b) \textbf{Data summarization.} Model architecture for UCI Adult Census Dataset~\cite{ron1996uci}. \fan{how do we interpret the table?}
}\label{tab:acc}
\begin{tabular}{cccc}
\toprule \multirow{2}{*}{\textbf{\scriptsize Data Type}} & \textbf{\scriptsize Handwritten Digit}  & \textbf{\scriptsize Object} & \textbf{\scriptsize House Number}  \\ & \textbf{\scriptsize (Logistic Regression)}  & \textbf{\scriptsize (ResNet18)} & \textbf{\scriptsize (ResNet18)}  \\ 
% \midrule
% Original Model on Benign Data & 0.998 & 1 & 1   \\
\midrule
Benign Data & 0.998 & 0.981 & 1.000   \\ 
Watermark Data & 0.980 & 1.000 & 1.000   \\ \bottomrule
\end{tabular}
\end{table}

We additionally measure the prediction accuracy of the watermarked model on both benign and watermark instances and provide the results in Table~\ref{tab:acc}.
The results indicate that the amount of watermarks we added satisfies our purpose of claiming the ownership of the data source.

\subsection{Data Summarization}
\label{appendix:sec_ds}

For the experiment on UCI Adult Census dataset~\cite{ron1996uci} introduced in Section 4.3, we train the same multilayer perceptron model as~\cite{chen2018differentially}. 
% The network architecture is displayed in Table~\ref{tab:arch}.

% \begin{wrapfigure}{r}{0.3\textwidth}
% \vspace{-5mm}
% \begin{center}
% \begin{subfigure}{0.3\textwidth}
% \centering
% \includegraphics[width=\columnwidth]{figs/New_DS_Tiny_H2L.pdf}
% \end{subfigure}
% \end{center}
% \vspace{-5mm}
% \end{wrapfigure}

We consider another setting for this application: a ResNet-18 trained on Tiny ImageNet. In this setting, we use $95000$ points as the training set, $5000$ points to calculate the scores, and another $10000$ points as the held-out validation set.
In Fig.~\ref{fig:ds-tiny}, we plot the change of prediction accuracy (in percentage) with the change of the fraction of data removed (in percentage).
As it reveals, $K$NN-Shapley is able to maintain model performance even after removing $40\%$ of the whole training set. However, TMC-Shapley, G-Shapley, and LOO cannot finish in $24$ hours and hence are omitted from the figure.

In this experiment, we fine-tune the pretrained ResNet18 from He \textit{et al.}.
We train the ResNet18 with $15$ epochs and learning rate $0.001$ with SGD optimizer~\cite{kiefer1952sgd} and the model achieves an accuracy of $77.95$\% on the training set. Then, we extract the deep features of the training set and calculate their Shapley values. When evaluating the model performance on the summarized dataset, we re-train the ResNet18 with $30$ epochs and learning rate $0.01$.

\begin{figure}[htbp]
\begin{center}
\begin{subfigure}{0.31\textwidth}
\centering
\includegraphics[width=\columnwidth]{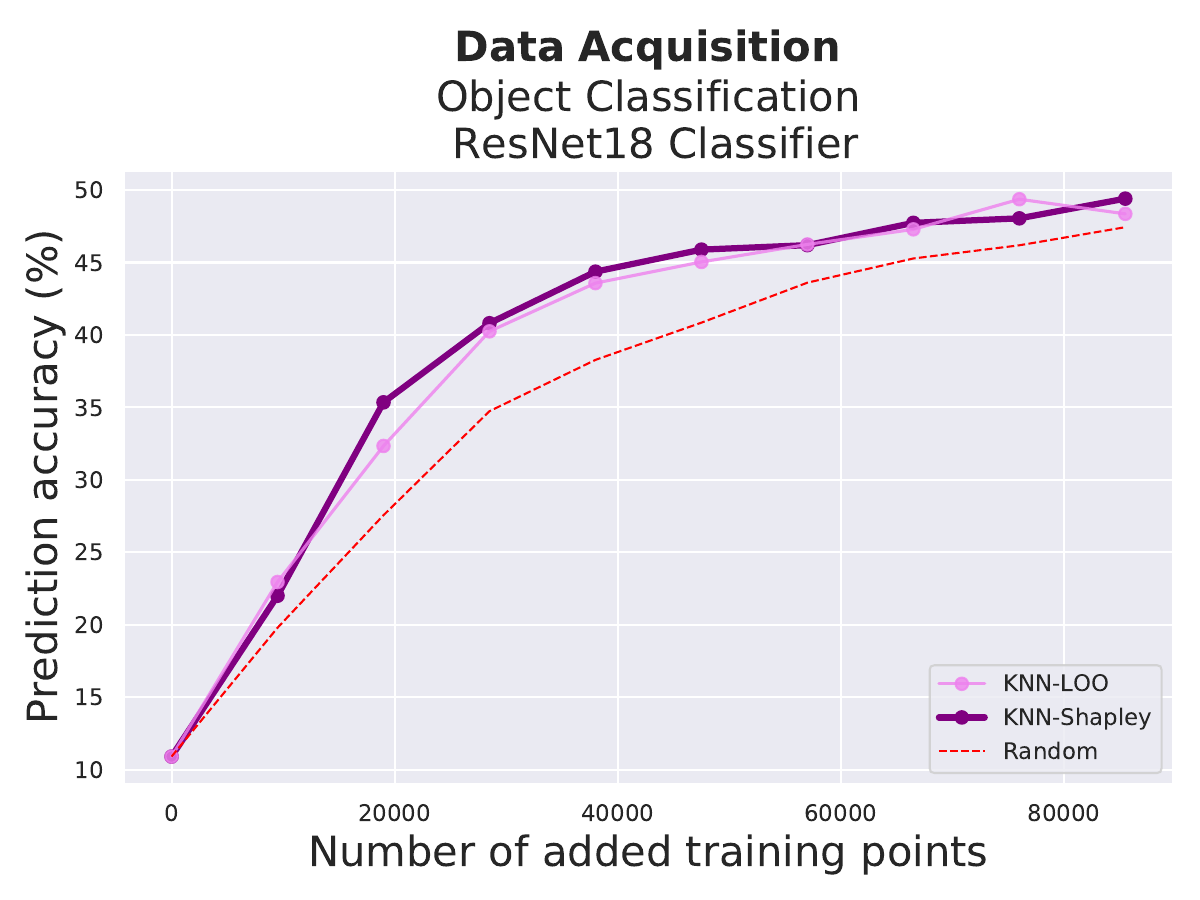}
\caption{}\label{fig:da-tiny}
\end{subfigure}
\begin{subfigure}{0.31\textwidth}
\centering
\includegraphics[width=\columnwidth]{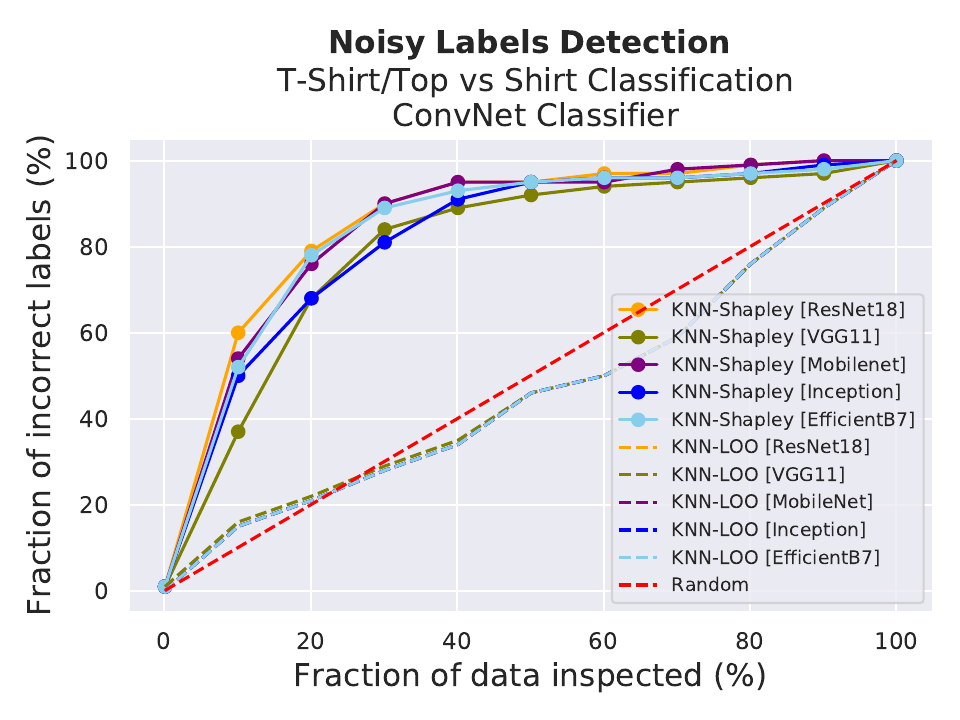}
\caption{}\label{fig:label-emb}
\end{subfigure}
    \begin{subtable}[H]{0.35\textwidth}
    \scriptsize
        \centering
\begin{tabular}{crr} 
\toprule \textbf{Method} & \textbf{SVHN} $\rightarrow{}$ \textbf{MNIST}  \\ 
& 
\includegraphics[width=5mm]{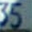}
\includegraphics[width=5mm]{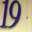} $\rightarrow{}$
\includegraphics[width=5mm]{./domainAdaptation_data/mnist2.png}
\includegraphics[width=5mm]{./domainAdaptation_data/mnist1.png}\\
\\\midrule
% \hline
$K$NN-Shapley & \textbf{9.65\% $\rightarrow{}$ 14.1\%}  \\ 
$K$NN-LOO &  9.65\% $\rightarrow{}$ 9.3\% \\ 
TMC-Shapley & ---  \\ 
LOO & --- \\
\bottomrule
\end{tabular}
      \caption{}
      \label{tab:adaptation}
\end{subtable}
\end{center}
\caption{\small (a) Data acquisition on Tiny ImageNet; (b) results of different embeddings of noisy label detection on Fashion-MNIST; (c) Domain adaptation on SVHN$\rightarrow$MNIST}
\end{figure}

\subsection{Data Acquisition}\label{appendix:sec_da}

% \begin{wrapfigure}{r}{0.3\textwidth}
% \vspace{-5mm}
% \begin{center}
% \begin{subfigure}{0.3\textwidth}
% \centering
% \includegraphics[width=\columnwidth]{figs/New_DA_Tiny_L2H.pdf}
% \end{subfigure}
% \end{center}
% \vspace{-5mm}
% \end{wrapfigure}

We follow the same protocols as in Section 4.3 to conduct the experiments on Tiny ImageNet, which in nature has realistic variation of data quality. We separate the training set into two parts with 5000 training points and 95000 new points. We calculate importannce of 2500 data points in the training set based on the other 2500 points.
In Fig.~\ref{fig:da-tiny} we plot the change of prediction accuracy with the number of added training points. Evidently, new data selected based on $K$NN-Shapley value improves model accuracy faster than all other methods.

\subsection{Domain Adaptation}
\label{appendix:sec_adaptation}
In section 4.3 we elaborated on the transfer between MNIST and USPS\footnote{\footnotesize from \url{https://www.kaggle.com/bistaumanga/usps-dataset}}, where we trained a multinomial logistic regression classifier. Here, we introduce another experiment on transferring from SVHN to MNIST. In this experiment, we train a ResNet18 model using $15$ epochs and learning rate $0.001$ with SGD optimizer on SVHN, since multinomial logistic regression is too simple to perform well in this setting. We pick $2000$ training data from SVHN, train a ResNet-18 model, and evaluate the performance on the whole test set of MNIST. $K$NN-Shapley is able to work on data of this scale efficiently while TMC-Shapley algorithm simply cannot finish in 48 hours. As shown in Table~\ref{tab:adaptation}, our KNN-Shapley achieves better performance than KNN-LOO.

\section{Impact of Different Embeddings}
\label{appendix:sec_embedding}

% \begin{wrapfigure}{r}{0.3\textwidth}
% \vspace{-10mm}
% \begin{center}
% \includegraphics[width=0.3\textwidth]{figs/Label_embedding.pdf}
% \end{center}
% \vspace{-5mm}
% \end{wrapfigure}

In Section 4.3 we provide the result corresponding to the embedding extracted by one single feature extractor for each dataset. In this section, for all the aforementioned experiments, we tried different embeddings extracted using five pre-trained classifiers including ResNet18, VGG11~\cite{Simonyan15vgg}, MobileNet~\cite{howard2017mobilenets}, Inception-V3~\cite{szegedy2015inception}, and EfficientNet B7~\cite{tan2019efficientnet}.

Illustrated in Fig.~\ref{fig:label-emb} is the comparison of two data importance measures: KNN-Shapley and KNN-LOO, each applied to five different embeddings. This experiment is carried out on the Fashion-MNIST dataset for the task of detecting noisy labels.
Notably, the five curves of KNN-Shapley are close to each other, and the same trend can also be observed for the five curves of KNN-LOO. Apart from this observation, the scores given by KNN-LOO are roughly the same as random, while our KNN-Shapley are all much higher. As a conclusion, the influence induced by using different embeddings is marginal compared to using different measures. Furthermore, our KNN-Shapley data importance measure can achieve terrific performance without the need of carefully selecting embeddings. We provide a comprehensive set of results in Fig.~\ref{fig:embedding}, where similar conclusions can be drawn.

As a supplement to Fig. 3 in the main body, similarly, in Fig.~\ref{fig:ds-viz-rest}, we provide the top 20 images with highest Shapley value, as well as the top 50 classes after the summarization step for each of the following embeddings: Resnet18, Inception-V3, and EfficientNet B7. As can be observed, there is a large range of overlap among the top classes for all these embeddings, which we believe is an intriguing phenomenon to study and will inspire future research.

\begin{figure}
\begin{center}
\begin{subfigure}{0.3\textwidth}
\centering
\includegraphics[width=\columnwidth]{figs/Poisoning_embedding_a.pdf}
\end{subfigure}
\begin{subfigure}{0.3\textwidth}
\centering
\includegraphics[width=\columnwidth]{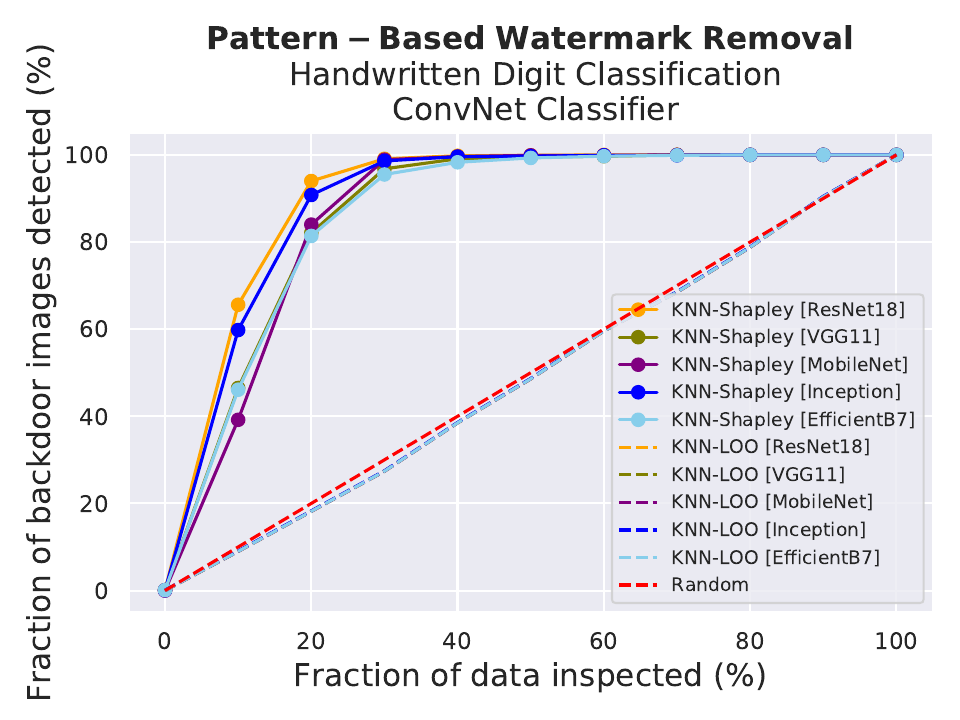}
\end{subfigure}
\begin{subfigure}{0.3\textwidth}
\centering
\includegraphics[width=\columnwidth]{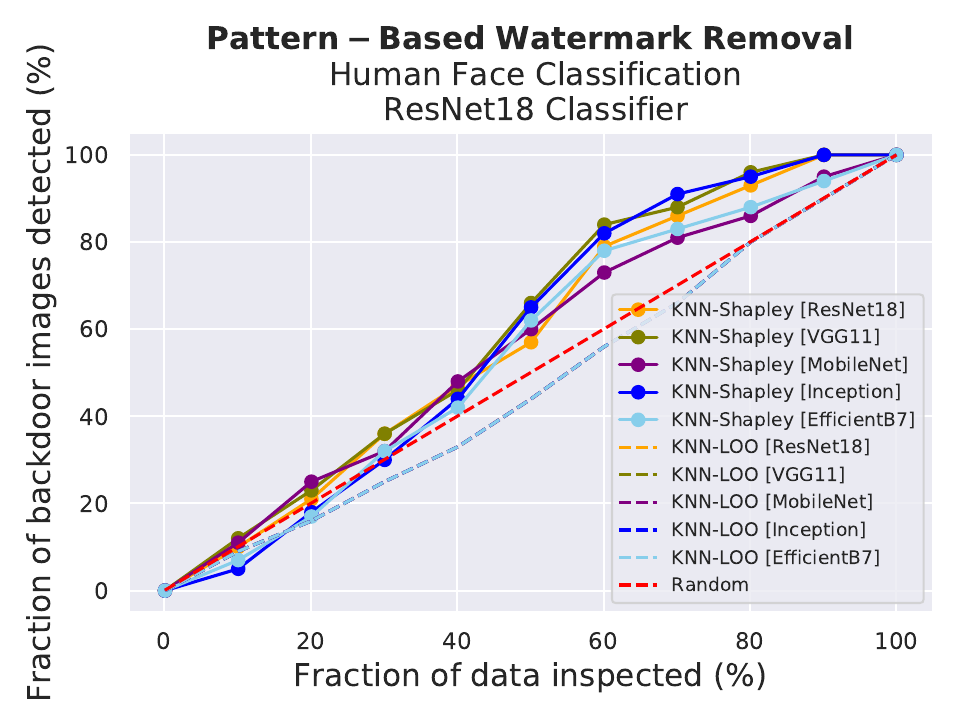}
\end{subfigure}
\begin{subfigure}{0.3\textwidth}
\centering
\includegraphics[width=\columnwidth]{figs/Watermarking_embedding_a.pdf}
\end{subfigure}
\begin{subfigure}{0.3\textwidth}
\centering
\includegraphics[width=\columnwidth]{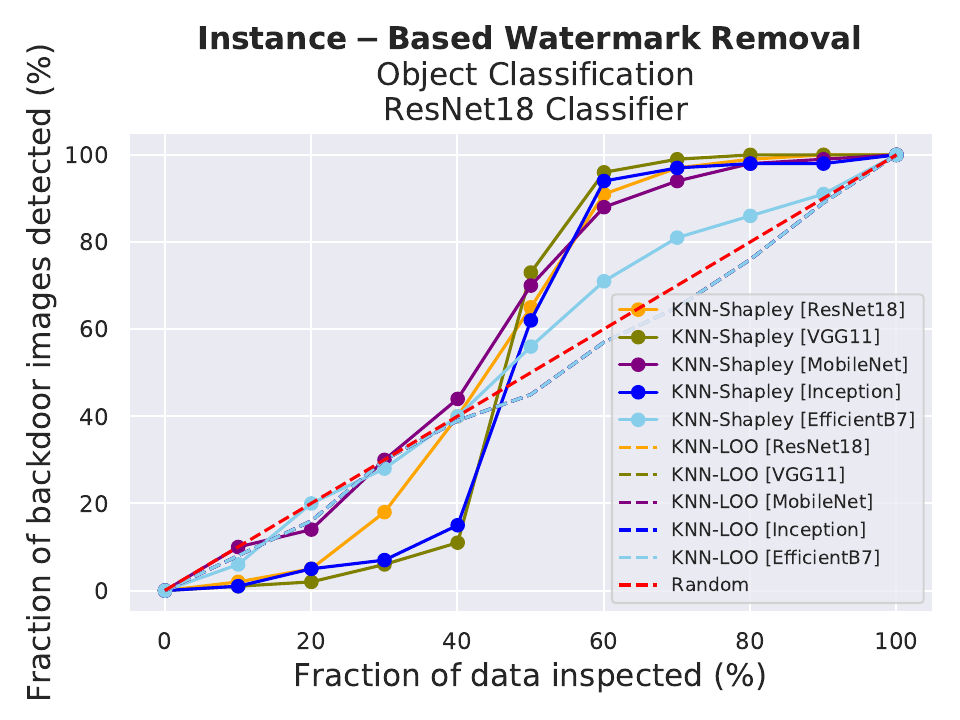}
\end{subfigure}
\begin{subfigure}{0.3\textwidth}
\centering
\includegraphics[width=\columnwidth]{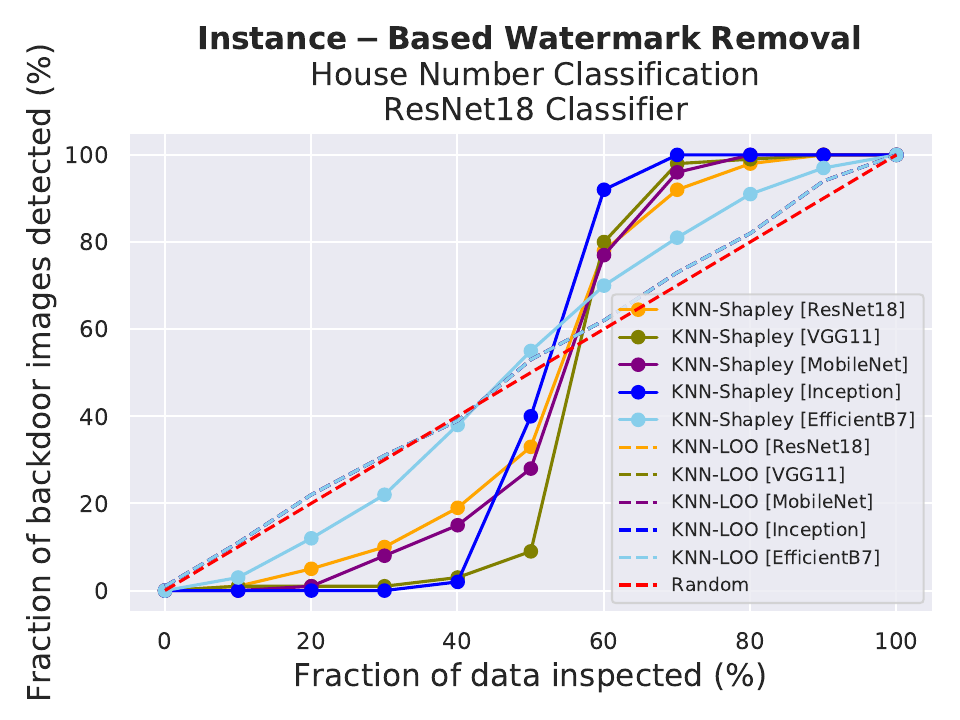}
\end{subfigure}
\begin{subfigure}{0.3\textwidth}
\centering
\includegraphics[width=\columnwidth]{figs/Compare_DS_Tiny_H2L.pdf}
\end{subfigure}
\begin{subfigure}{0.3\textwidth}
\centering
\includegraphics[width=\columnwidth]{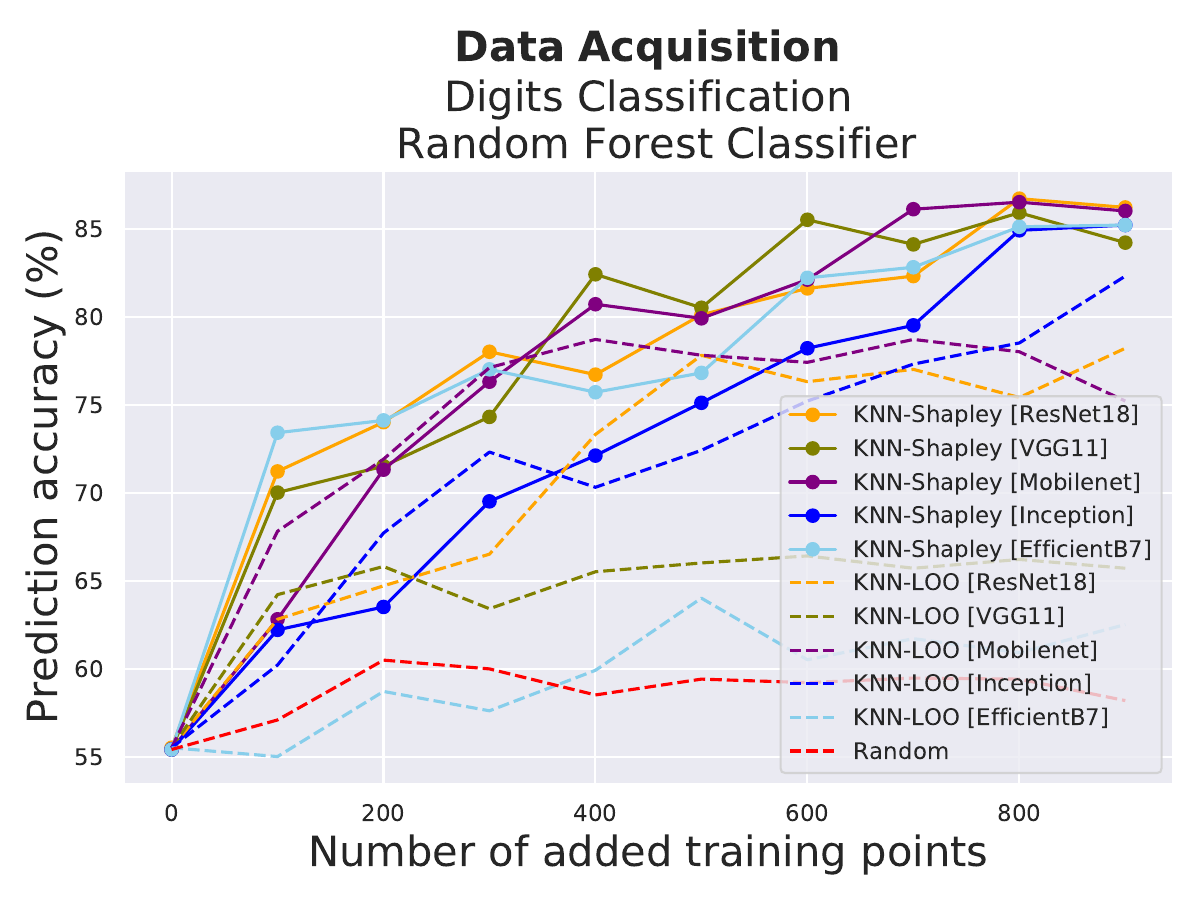}
\end{subfigure}
\begin{subfigure}{0.3\textwidth}
\centering
\includegraphics[width=\columnwidth]{figs/Compare_DA_Tiny_L2H.pdf}
\end{subfigure}
\end{center}
\caption{\small Comparisons of different embeddings on different datasets and different applications}\label{fig:embedding}
\end{figure}

\begin{figure}[H]
\begin{subfigure}{0.3\textwidth}
\centering
\includegraphics[width=\linewidth]{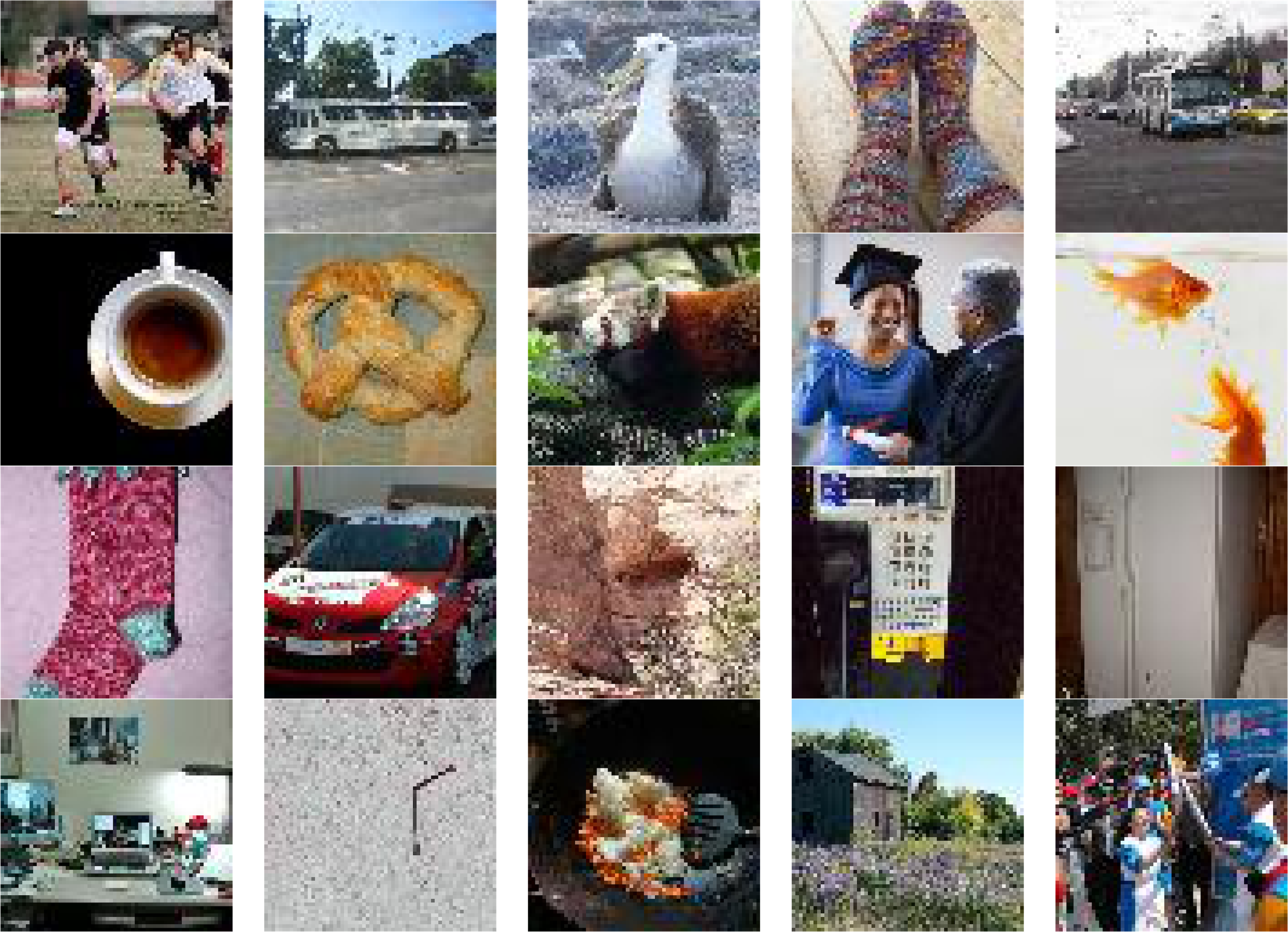}
\end{subfigure}
\hfill
\begin{subfigure}{0.3\textwidth}
\centering
\includegraphics[width=\linewidth]{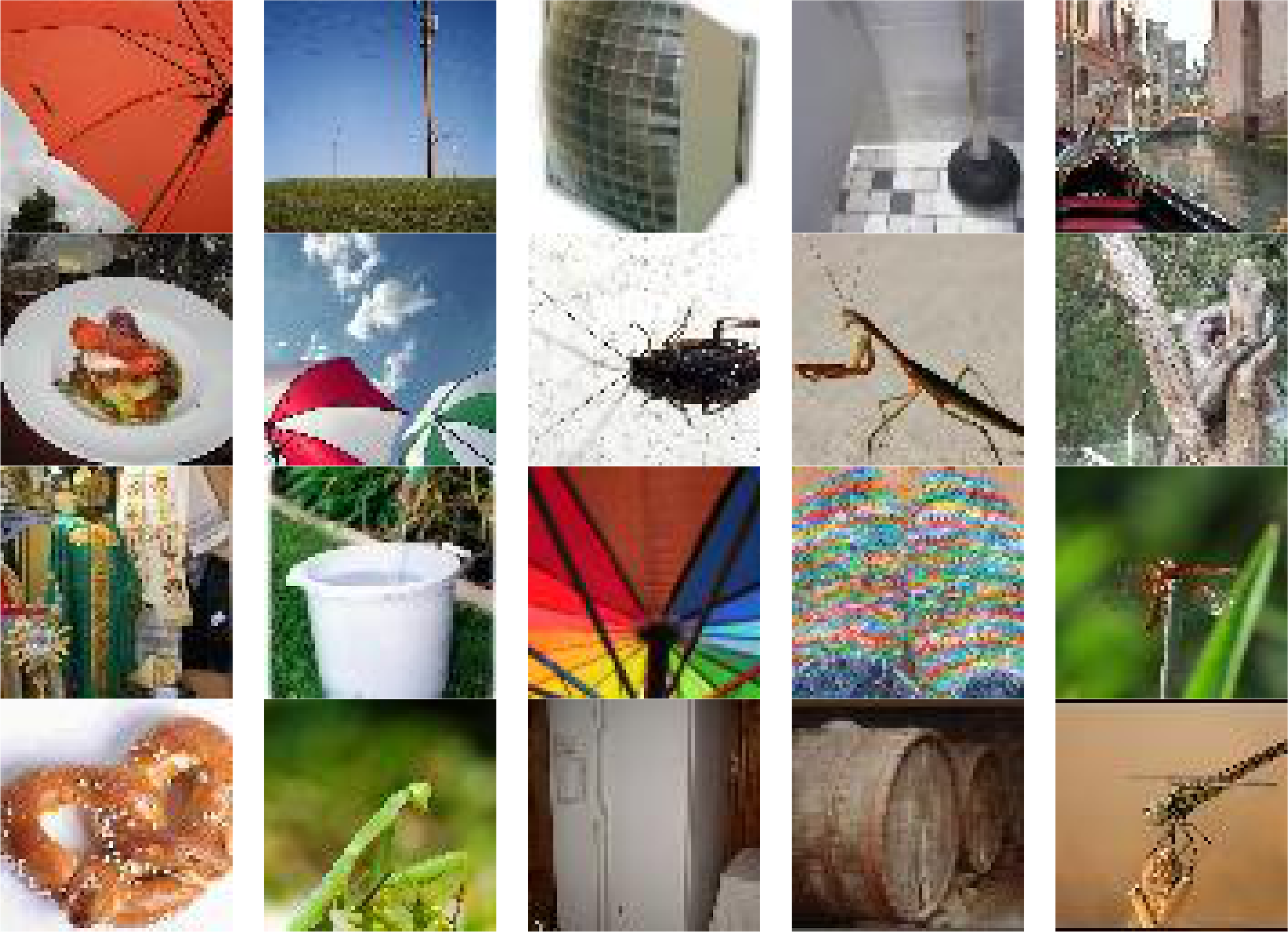}
\end{subfigure}%
\hfill
\begin{subfigure}{0.3\textwidth}%
\centering
\includegraphics[width=\linewidth]{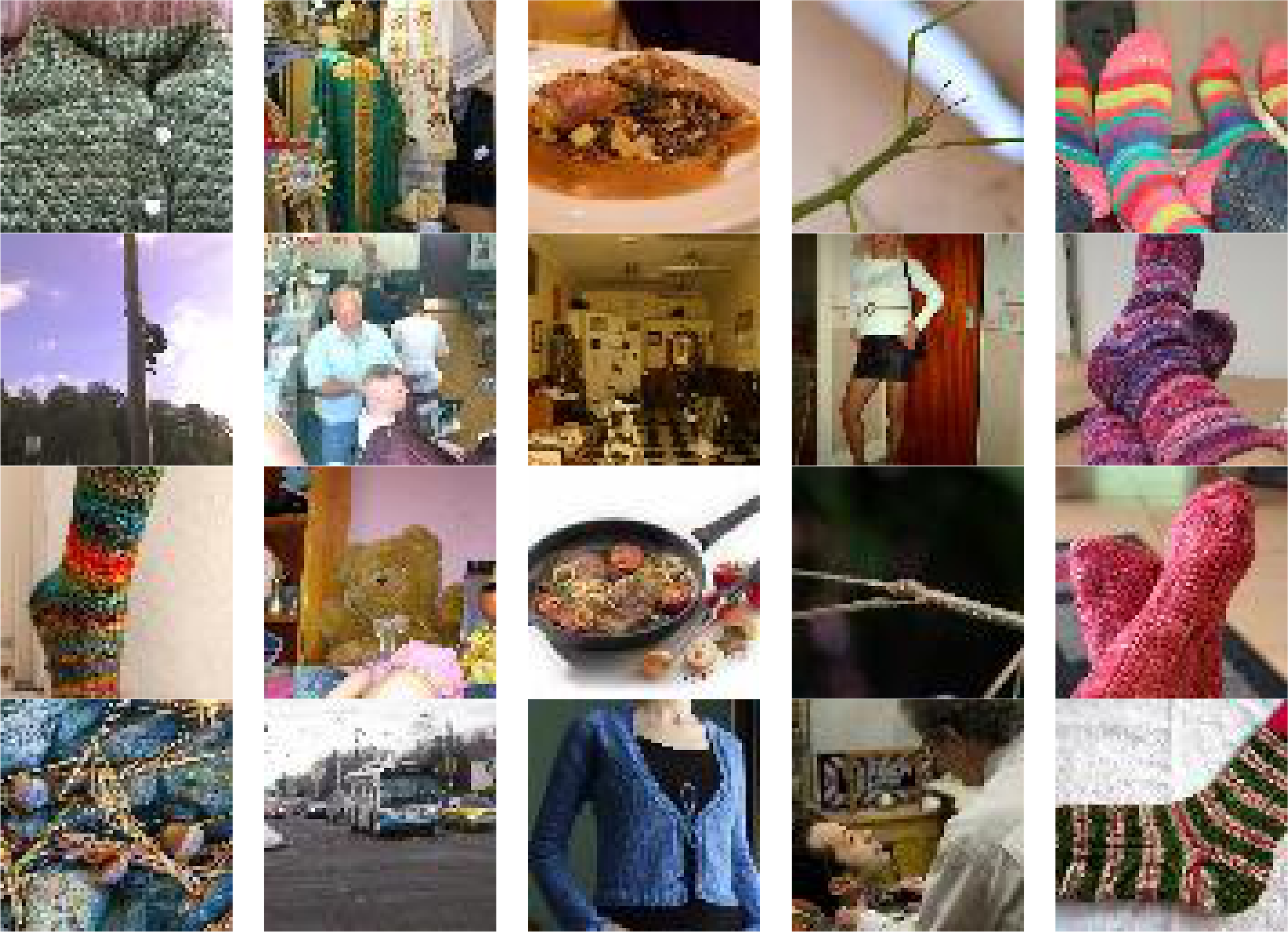}
\end{subfigure}%
\hfill

\begin{subfigure}{0.3\textwidth}
\centering
\includegraphics[width=\linewidth]{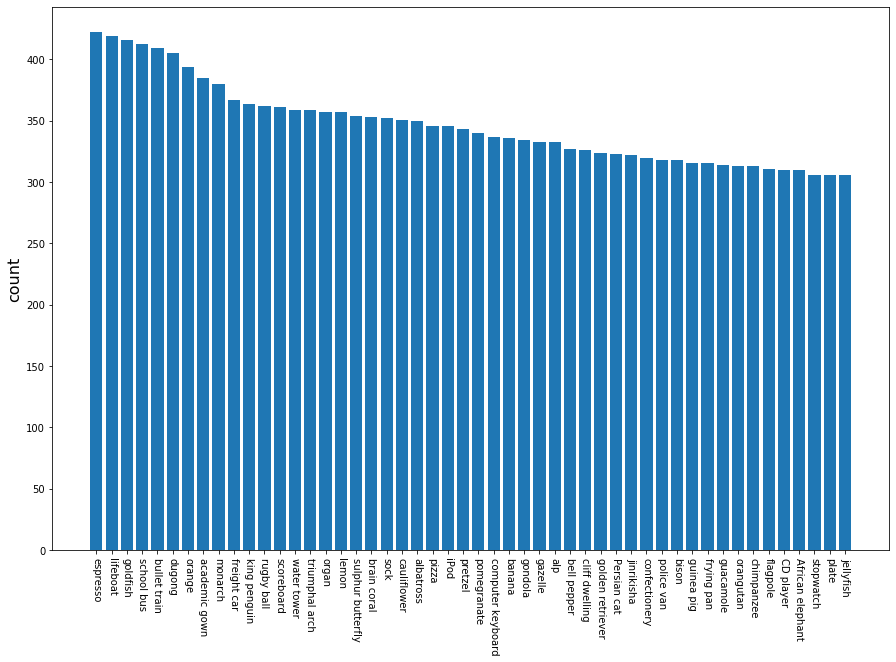}
\end{subfigure}%
\hfill
\begin{subfigure}{0.3\textwidth}
\centering
\includegraphics[width=\linewidth]{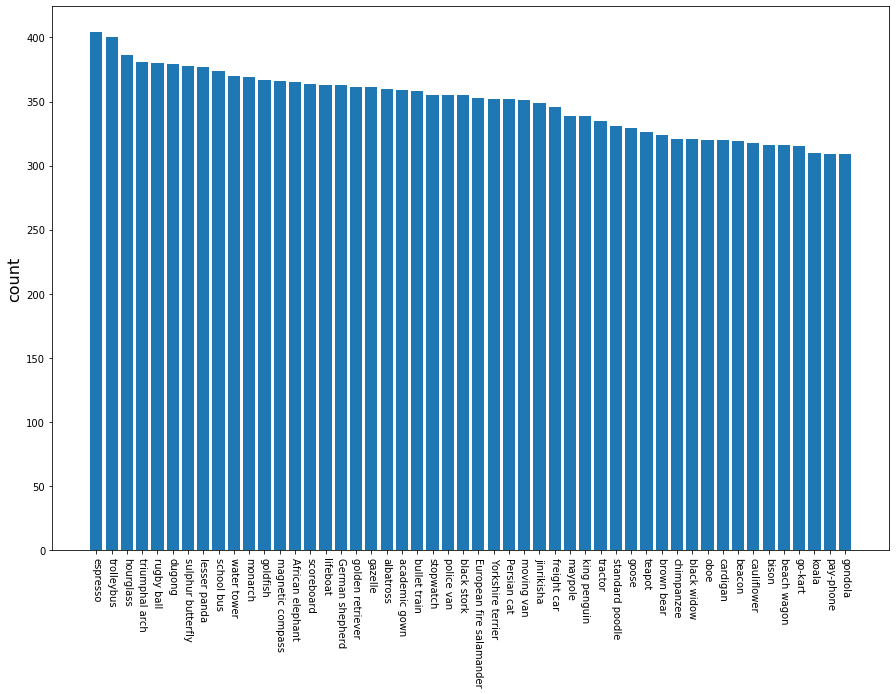}
\end{subfigure}%
\hfill
\begin{subfigure}{0.3\textwidth}
\centering
\includegraphics[width=\linewidth]{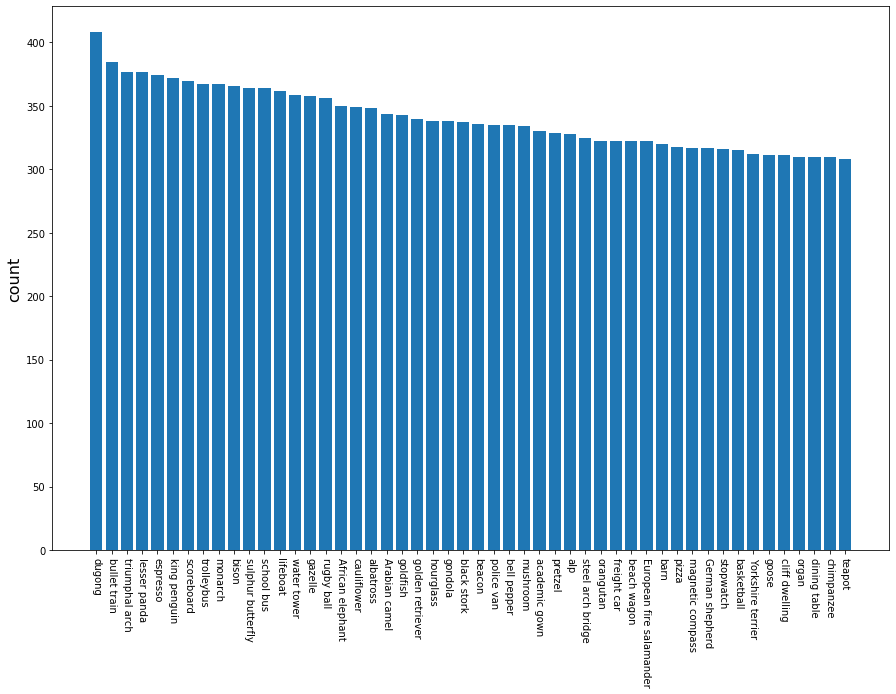}
\end{subfigure}%
% \hfill
\caption{\small First row: top 20 selected images with the highest Shapley values in Tiny ImageNet for Resnet18, Inception-V3, and EfficientNet B7 embeddings, respectively;
% (b) (d) statistics for the class distribution of the top 50 classes with the most images chosen. 
Second row: counts of images in top 50 classes after the summarization step (sorted by the count in a decreasing manner). There are many overlapped classes among different embeddings. 
}
\label{fig:ds-viz-rest}
\end{figure}

% \paragraph{test}